\documentclass[letterpaper]{article} % DO NOT CHANGE THIS
\usepackage{aaai24}  % DO NOT CHANGE THIS
\usepackage{times}  % DO NOT CHANGE THIS
\usepackage{helvet}  % DO NOT CHANGE THIS
\usepackage{courier}  % DO NOT CHANGE THIS
\usepackage[hyphens]{url}  % DO NOT CHANGE THIS
\usepackage{graphicx} % DO NOT CHANGE THIS
\usepackage{natbib}  % DO NOT CHANGE THIS AND DO NOT ADD ANY OPTIONS TO IT
\usepackage{caption} % DO NOT CHANGE THIS AND DO NOT ADD ANY OPTIONS TO IT
\usepackage[ruled,vlined,linesnumbered]{algorithm2e} 
\usepackage{newfloat}
\usepackage{listings}
\DeclareCaptionStyle{ruled}{labelfont=normalfont,labelsep=colon,strut=off} % DO NOT CHANGE THIS
\usepackage[dvipsnames]{xcolor}

\usepackage{amssymb,amsthm}
\usepackage{diagbox}
\newtheorem{example}{Example}
\newtheorem{theorem}{Theorem}
\newtheorem{corollary}{Corollary}
\newtheorem{lemma}{Lemma}
\newtheorem{definition}{Definition}
\newtheorem{proposition}{Proposition}
\newtheorem{remark}{Remark}
\newtheorem{assumption}{Assumption}

\newtheorem*{proof sketch}{Proof sketch}

\usepackage{multirow}

\usepackage[utf8]{inputenc}
\usepackage[T1]{fontenc}    
\usepackage{url}            
\usepackage{booktabs}       
\usepackage{amsfonts}       
\usepackage{nicefrac}      
\usepackage{microtype}      

\usepackage{times}
\usepackage{soul}
\usepackage{fdsymbol}
\usepackage{graphicx}
\usepackage{amsmath}

\usepackage{import}
\usepackage{subfigure}

\definecolor{edgecolor}{RGB}{215,215,215}

\usepackage{enumitem}
\usepackage{tikz}
\usetikzlibrary{arrows,shapes,positioning,shadows,trees}
\usetikzlibrary{decorations.pathmorphing}
\let\bot\Vbar 
\newdimen\arrowsize
\pgfarrowsdeclare{arcsq}{arcsq}
{
	\arrowsize=0.15pt
	\advance\arrowsize by .5\pgflinewidth
	\pgfarrowsleftextend{-4\arrowsize-.5\pgflinewidth}
	\pgfarrowsrightextend{.5\pgflinewidth}
}
{
	\arrowsize=0.65pt
	\advance\arrowsize by .5\pgflinewidth
	\pgfsetdash{}{0pt} % do not dash
	\pgfsetroundjoin   % fix join
	\pgfsetroundcap    % fix cap
	\pgfpathmoveto{\pgfpoint{0\arrowsize}{0\arrowsize}}
	\pgfpatharc{-90}{-140}{4\arrowsize}
	\pgfusepathqstroke
	\pgfpathmoveto{\pgfpointorigin}
	\pgfpatharc{90}{140}{4\arrowsize}
	\pgfusepathqstroke
}    

\usetikzlibrary{positioning, fit, calc} 
\usepackage{xcolor}
\tikzset{block/.style={draw, thick, text width=2cm , minimum height=1.3cm, align=center}, 
	line/.style={-latex}   
}

\labelformat{equation}{(#1)}

\title{Identification of Causal Structure in the Presence of Missing Data with Additive Noise Model}

\author{
Jie Qiao\textsuperscript{\rm 1},
Zhengming Chen\textsuperscript{\rm 1,3},
Jianhua Yu\textsuperscript{\rm 1},
Ruichu Cai\textsuperscript{\rm 1,2}\thanks{Corresponding author.},
Zhifeng Hao\textsuperscript{\rm 4}
}
\affiliations{
\textsuperscript{\rm 1}School of Computer Science, Guangdong University of Technology, Guangzhou 510006, China\\
\textsuperscript{\rm 2}Peng Cheng Laboratory, Shenzhen 518066, China\\
\textsuperscript{\rm 3}Machine Learning Department, Mohamed bin Zayed University of Artificial Intelligence, Abu Dhabi, UAE \\
\textsuperscript{\rm 4}College of Science, Shantou University, Shantou 515063, China \\
\{qiaojie.chn, chenzhengming1103, philoso521, cairuichu\}@gmail.com, haozhifeng@stu.edu.cn
}

\begin{document}

\maketitle

\begin{abstract}
	
	Missing data are an unavoidable complication frequently encountered in many causal discovery tasks. 
	When a missing process depends on the missing values themselves (known as \textit{self-masking missingness}), the recovery of the joint distribution becomes unattainable, and detecting the presence of such self-masking missingness remains a perplexing challenge. Consequently, due to the inability to reconstruct the original distribution and to discern the underlying missingness mechanism, simply applying existing causal discovery methods would lead to wrong conclusions. In this work, we found that the recent advances additive noise model has the potential for learning causal structure under the existence of the self-masking missingness. With this observation, we aim to investigate the identification problem of learning causal structure from missing data under an additive noise model with different missingness mechanisms, where the `no self-masking missingness' assumption can be eliminated appropriately. 
	Specifically, we first elegantly extend the scope of identifiability of causal skeleton to the case with weak self-masking missingness (i.e., no other variable could be the cause of self-masking indicators except itself). We further provide the sufficient and necessary identification conditions of the causal direction under additive noise model and show that the causal structure can be identified up to an IN-equivalent pattern. We finally propose a practical algorithm based on the above theoretical results on learning the causal skeleton and causal direction. Extensive experiments on synthetic and real data demonstrate the efficiency and effectiveness of the proposed algorithms.
	
\end{abstract}

\section{Introduction}
Missing data are ubiquitous in many fields and causal discovery from missing observational data is challenging due to possible complex \textit{missingness mechanisms}\textemdash causal structure among the causal variables and its missingness indicators. "indicator' is variables that specify whether a value was missing for an observation. With different structures of the missingness mechanisms, following \citet{rubin1976inference}, the missingness types can be categorized as Missing Completely At Random (MCAR), Missing At Random (MAR), and Missing Not At Random (MNAR). 

When data are MCAR, the causal variables, and their missingness indicators are independent such that the missingness mechanism is ignorable, and one can perform the listwise deletion that simply drops the samples with missing value \cite{strobl2018fast}. However, for MNAR, the missingness mechanism is not ignorable, and the simple deletion would introduce bias resulting in incorrect inference \cite{rubin2004multiple,mohan2013graphical,tu2019causal}. Much effort has been made in studying the recoverability of the MNAR data to correct the spurious correlations from deleted-wised distribution using reweighting by incorporating prior information of the causal graph and missingness mechanisms \cite{bhattacharya2020identification,mohan2021graphical,nabi2020full}. However, when a missing process depends on the missing values themselves (known as \textit{self-masking} belonging to MNAR), it suffers from a serious identifiability issue such that the origin distribution is not recoverable \cite{nabi2020full}.

Existing works, therefore, mainly focus on causal discovery with no self-masking missingness mechanisms. On one hand, \cite{gao2022missdag} addresses the M(C)AR cases using the identifiable additive noise models (ANMs) \cite{hoyer2008nonlinear,shimizu2006linear,peters2014identifiability} with a general EM-based framework to perform causal discovery in the presence of missing data. On the other hand, to learn causal structure in MNAR, some constraints of the missingness mechanisms must be made. With the no self-masking assumption, \citet{gain2018structure} proposes to extend the constraint-based PC algorithm \cite{spirtes2000causation} by correcting the bias of each conditional independence (CI) test bought by the missing value. While, instead of correcting CI test, \citet{tu2019causal} proposes a more efficient post-correction strategy to correct the spurious edges produced by the original PC algorithm. 

However, due to the limited recoverability, it is still unclear how to identify causal structure in the presence of the self-masking missingness mechanism, which is believed to be the most commonly encountered in practice \cite{osborne2013best}, e.g., smokers do not answer any questions about their smoking behavior in insurance applications, and people with very high or low income do not disclose their income \cite{mohan2018handling}. 

In this work, we found that the recent advances additive noise model has the potential for learning causal structure under the existence of the self-masking missingness mechanism. Take Fig. \ref{fig:anm_selfmasking_example} as an example, suppose a causal pair $R_X\leftarrow X\to Y$ where $X$ is a self-masking missing variable and $R_X$ is the missingness indicator of $X$. Although we can only access the distribution of $P(X|R_X=0)$, the causal direction is still identifiable since the independence noise $X\Vbar \epsilon_Y$ still holds\textemdash considering each separate part follows an identifiable ANM. In contrast, suppose $X\to Y \to R_Y$ as shown in Fig. \ref{fig:illust_simple_example}(a), unfortunately, ANM is not identifiable and the independence noise will not hold in this case since we can only access the distribution in $R_Y=0$ such that $X \not\Vbar \epsilon_Y|R_Y$ ($R_Y$ is the descendent of the collider $Y$). Nevertheless, if the self-masking of $Y$ is known, we can still identify such causal direction since we know that such dependence noise is brought by $R_Y$, otherwise, the causal direction can be identified in the first place. With this observation, in this paper, we aim to address the following two questions: 1) How to identify the self-masking missingness mechanism in learning causal skeleton? 2) How to identify the causal direction when ANM is not identifiable due to the self-masking missingness?

In answering the first question, we relax the assumptions in MVPC \cite{tu2019causal} to allow the existence of the weak self-masking missingness and propose a post-correction strategy to correct the spurious edges produced by the self-masking missingness. 
For the second question, by giving the causal skeleton, we propose an equivalent class for ANM and develop a searching method to identify the causal direction from the equivalent class.
Our main contributions are summarized as follows: 1) we propose a practical causal skeleton learning algorithm for learning the skeleton among causal variables and the missingness indicators in the presence of weak self-masking missingness. 2) We provide a theoretical analysis of the identification of the additive noise model under missing data. 3) Combine with the causal skeleton and additive noise model, we show that the causal structure can be identified up to an IN-equivalent pattern. 4) We provide some insightful orientation rules to further orient the IN-equivalent pattern.

\begin{figure}
	\centering
	\subfigure{
		\label{fig:anm_selfmasking_example_a}
		\includegraphics[width=0.225\textwidth]{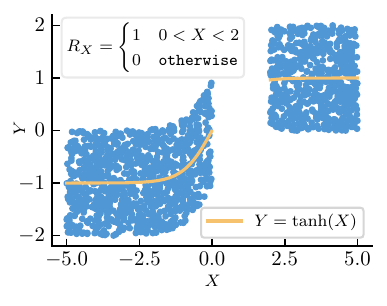}}
	\subfigure{
		\label{fig:anm_selfmasking_example_b}
		\includegraphics[width=0.225\textwidth]{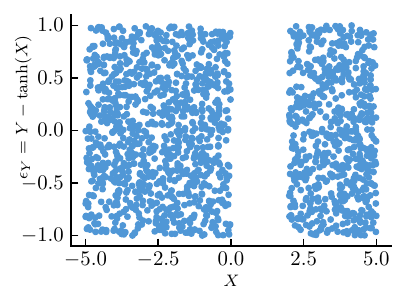}}
	%\vspace{-0.5cm}
	\caption{Illustration of self-masking missingness ANM in which $X$ is the self masking missing variable satisfying $R_x\leftarrow X\rightarrow Y$. Panels left show the scatter plot of $X$ and $Y=\tanh(X)+N$ and Panels right show the corresponding noise of ANM.}
	\label{fig:anm_selfmasking_example}
\end{figure}
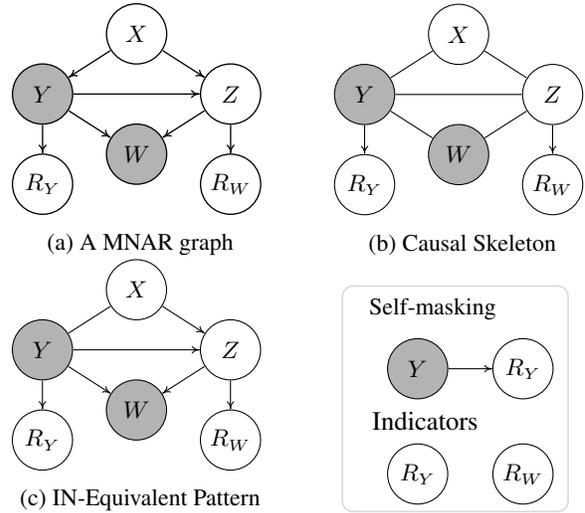
\begin{figure}[t]
	\center
	\begin{tikzpicture}[scale=1.0, line width=0.5pt, inner sep=0.2mm, shorten >=.1pt, shorten <=.1pt]
		
		\draw (1.75, 0.8) node(X) [circle, minimum size=0.8cm, draw] {{\footnotesize\,$X$\,}};

		\draw (0.5, 0) node(Y) [circle, fill=gray!60, minimum size=0.8cm, draw] {{\footnotesize\,$Y$\,}};
		
		\draw (3, 0) node(Z) [circle, minimum size=0.8cm, draw] {{\footnotesize\,$Z$\,}};
		
		\draw (1.75, -0.8) node(W) [circle,fill=gray!60, minimum size=0.8cm, draw] {{\footnotesize\,$W$\,}};
		
		%For Ry
		\draw (0.5, -1.2) node(Ry) [circle, minimum size=0.8cm,draw] {{\footnotesize\,$R_Y$\,}};
		
		%For Rw
		
		\draw (3, -1.2) node(Rw) [circle,minimum size=0.8cm, draw] {{\footnotesize\,$R_W$\,}};

		\draw[-arcsq] (X) -- (Y) node[pos=0.5,sloped,above] {};
		\draw[-arcsq] (X) -- (Z) node[pos=0.5,sloped,above] {};
		
		\draw[-arcsq] (Y) -- (W) node[pos=0.5,sloped,above] {}; 
		\draw[-arcsq] (Z) -- (W) node[pos=0.5,sloped,above] {}; 
		\draw[-arcsq] (Y) -- (Z) node[pos=0.5,sloped,above] {};
		\draw[-arcsq] (Y) -- (Ry) node[pos=0.5,sloped,above] {};
		\draw[-arcsq] (Z) -- (Rw) node[pos=0.5,sloped,above] {};
		
		\draw (1.8, -2) node(ii) [] {{\footnotesize\,(a) A MNAR graph\,}};
	\end{tikzpicture}~~~~~~~~~~~~
	\begin{tikzpicture}[scale=1.0,inner sep=0.2mm, shorten >=.1pt, shorten <=.1pt]
		% \begin{tikzpicture}[scale=1.2, line width=0.5pt, inner sep=0.2mm, shorten >=.1pt, shorten <=.1pt]
			\draw (1.75, 0.8) node(X) [circle, minimum size=0.8cm, draw] {{\footnotesize\,$X$\,}};

			\draw (0.5, 0) node(Y) [circle, fill=gray!60, minimum size=0.8cm, draw] {{\footnotesize\,$Y$\,}};
			
			\draw (3, 0) node(Z) [circle, minimum size=0.8cm, draw] {{\footnotesize\,$Z$\,}};
			
			\draw (1.75, -0.8) node(W) [circle,fill=gray!60, minimum size=0.8cm, draw] {{\footnotesize\,$W$\,}};
			
			%For Ry
			\draw (0.5, -1.2) node(Ry) [circle,minimum size=0.8cm, draw] {{\footnotesize\,$R_Y$\,}};
			
			%For Rw
			
			\draw (3, -1.2) node(Rw) [circle, minimum size=0.8cm,draw] {{\footnotesize\,$R_W$\,}};

			\draw[-] (Y) -- (X) node[pos=0.5,sloped,above] {};
			\draw[-] (X) -- (Z) node[pos=0.5,sloped,above] {};
			
			\draw[-] (Y) -- (W) node[pos=0.5,sloped,above] {}; 
			\draw[-] (Z) -- (W) node[pos=0.5,sloped,above] {}; 
			\draw[-] (Y) -- (Z) node[pos=0.5,sloped,above] {};
			\draw[-arcsq] (Y) -- (Ry) node[pos=0.5,sloped,above] {};
			\draw[-arcsq] (Z) -- (Rw) node[pos=0.5,sloped,above] {};
			\draw (1.8, -2) node(ii) [] {{\footnotesize\,(b) Causal Skeleton\,}};
			
		\end{tikzpicture}~~~~~~~~~~~~\\
		\begin{tikzpicture}[scale=1.0,inner sep=0.2mm, shorten >=.1pt, shorten <=.1pt]
			\draw (1.75, 0.8) node(X) [circle, minimum size=0.8cm, draw] {{\footnotesize\,$X$\,}};

			\draw (0.5, 0) node(Y) [circle, fill=gray!60, minimum size=0.8cm, draw] {{\footnotesize\,$Y$\,}};
			
			\draw (3, 0) node(Z) [circle, minimum size=0.8cm, draw] {{\footnotesize\,$Z$\,}};
			
			\draw (1.75, -0.8) node(W) [circle,fill=gray!60, minimum size=0.8cm, draw] {{\footnotesize\,$W$\,}};
			
			%For Ry
			\draw (0.5, -1.2) node(Ry) [circle,minimum size=0.8cm, draw] {{\footnotesize\,$R_Y$\,}};
			
			%For Rw
			
			\draw (3, -1.2) node(Rw) [circle,minimum size=0.8cm, draw] {{\footnotesize\,$R_W$\,}};

			\draw[-] (X) -- (Y) node[pos=0.5,sloped] {};
			\draw[-arcsq] (X) -- (Z) node[pos=0.5,sloped,above] {};
			
			\draw[-arcsq] (Y) -- (W) node[pos=0.5,sloped,above] {}; 
			\draw[-arcsq] (Z) -- (W) node[pos=0.5,sloped,above] {}; 
			\draw[-arcsq] (Y) -- (Z) node[pos=0.5,sloped,above] {};
			\draw[-arcsq] (Y) -- (Ry) node[pos=0.5,sloped,above] {};
			\draw[-arcsq] (Z) -- (Rw) node[pos=0.5,sloped,above] {};
			\draw (1.8, -2) node(ii) [] {{\footnotesize\,(c) IN-Equivalent Pattern\,}};
		\end{tikzpicture}~~~~~~~~~~~~
		\begin{tikzpicture}[scale=1.0,inner sep=0.2mm, shorten >=.1pt, shorten <=.1pt]
			
			%\draw   (80,87) -- (150,87) -- (150,127) -- (80,127) -- cycle ;
			% \draw   (0,1) -- (2,0) -- (2,3) -- (0,3) -- cycle ;
			%\draw (-1,-1) rectangle (1,1);
			\draw  [draw = edgecolor, shift = {(0.3,-0.2)},rounded corners,line width=0.2mm] (-1.5,-1.5) rectangle (1.5,1.5);
			
			\draw (0, 1.0) node(iii) [] {{\footnotesize\, Self-masking \,}};
			
			\draw (-0.2, 0.2) node(Y) [circle, fill=gray!60, minimum size=0.8cm, draw] {{\footnotesize\,$Y$\,}};
			
			%For Ry
			\draw (1.2, 0.2) node(Ry) [circle, minimum size=0.8cm,draw] {{\footnotesize\,$R_Y$\,}};
			
			%indicator
			\draw (-0.1, -0.5) node(ii) [] {{Indicators}};
			
			% 		 \draw (1.0, -1.2) node(V) [circle,minimum size=0.8cm, draw] {{\footnotesize\, $R_Y$}};
			
			% 		 \draw (2.0, -1.2) node(Vi) [circle,minimum size=0.8cm, draw] {{ \footnotesize\,$R_W$}};

			\draw (-0.2, -1.2) node(v_Ry) [circle,minimum size=0.8cm, draw] {{\footnotesize\,$R_Y$\,}};
			
			%For Rw
			
			\draw (1.2, -1.2) node(v_Rw) [circle,minimum size=0.8cm, draw] {{\footnotesize\,$R_W$\,}};
			\draw[-arcsq] (Y) -- (Ry) node[pos=0.5,sloped,above] {};

		\end{tikzpicture}~~~~~~~~~~~~
		\label{fig:in_equivalent}
		% }
	\caption{An example of missingness graph with self-masking missingness. Here, gray nodes are partially observed variables, and white nodes are fully observed variables, $R_Y$ and $R_W$ are the missingness indicators of $Y$ and $W$, respectively.}  \label{fig:illust_simple_example} 
\end{figure}

\section{Problem Definition}\label{sec:problem}

In this paper, we focus on the problem of learning causal structure from missing data based on the additive noise model (ANM). We begin with the description of the missingness graphs (or \textit{m-graph}) \citep{mohan2021graphical} and the ANM, used as a representation of complex structures in missing data. Both of them are building blocks for our results.

\paragraph{Missingness Graph.} A missingness graph $G(\mathbb{V}, \mathbf{E})$ is a Directed Acyclic Graph (DAG) where $\mathbb{V}=\mathbf{V} \cup \mathbf{V}^* \cup \mathbf{R}$ is the set of nodes and $\mathbf{E}$ is the set of edges. 
$\mathbf{V}$ is a set of observable variables which can be decomposed into $\mathbf{V}=\mathbf{V}_o\cup \mathbf{V}_m$, where $\mathbf{V}_o $ is the set of fully observed variables, depicted as white nodes in our figure, and $\mathbf{V}_m$ is the set of partially observed variables that contain missing values, depicted as gray nodes in our figure. Each $\mathbf{V}_i^*\in \mathbf{V}^*$ is a proxy variable that is actually observed and $R_{V_i}\in \mathbf{R}$ is the missingness indicator of the partially observed variable such that
\begin{equation}
	V_{i}^{*} =\left\{
	\begin{array}{cl}
		V_{i} &   R_{V_{i}} =0\\
		? &  R_{V_{i}} =1\\
	\end{array} \right.
\end{equation}
where $R_{V_{i}} =1$ denotes the corresponding entry is missing while $R_{V_{i}} =0$ means the entry is observable and $V_{i}^*$ takes the value of $V_{i}$. 

Our work is in the framework of causal graphical models. Some concepts used here without explicit definition, such as d-separation, and Markov equivalent class, can be found in standard sources \cite{spirtes2000causation}. Besides, we use ``$\bot$'' to denote the independent relation in dataset and use ``$\bot_d$'' to denote the d-separation in m-graph. With the m-graph, we say data are MCAR if ${\mathbf{V}_m,\mathbf{V}_o}\bot_d \mathbf{R}$ holds and MAR if $\mathbf{V}_m\bot_d \mathbf{R}|\mathbf{V}_o$, and MNAR otherwise. We say a path $P=( V_{i_0} ,V_{i_1} ,\dots ,V_{i_{k}})$ in $G$ is a \textit{directed path} if it is a sequence of nodes of $G$ where there is a directed edge from $V_{i_j}$ to $V_{i_{(j+1)}}$ for any $0\leq j\leq k-1$. For simplifying graphical concepts, we use the following symbols $Pa_{V_i}=\{V_j|V_j \to V_i\}$, $Ch_{V_i}=\{V_j|V_i\to V_j\}$, $Anc_{V_i}=\{V_j|V_j \rightsquigarrow V_i\}$, $Des_{V_i}=\{V_j|V_i \rightsquigarrow V_j\}$, $Adj_{V_i}=\{V_j|V_j \leftarrow V_i \vee V_j \rightarrow V_i\}$ to denote the set of parents, children, ancestors, descendants, and adjacent nodes of $V_i$, respectively.

\begin{example}[Illustrative example for the m-graph]
	Consider the m-graph in Fig. \ref{fig:illust_simple_example}(a). The full observed variable set is $\mathbf{V}_o = \{X, Z\}$ and the partially observed variable set is $\mathbf{V}_m = \{Y^*, W^*\}$. Moreover, $R_Y$ and $R_W$ is the missing indicator of $Y$ and $W$ respectively, where $R_Y$ is a `self-masking missingness' indicator.
\end{example}

\paragraph{Additive Noise Model.}
Additive noise model assumes the observed data has been generated the following way: for each variable $V_i \in \mathbf{V}$ and its causal parents $Pa_{V_i} \subseteq \mathbf{V}$, there is
\begin{equation}\label{anm}
	V_i=f_i(Pa_{V_i})+\varepsilon_i,\quad Pa_{V_i}\Vbar \varepsilon_i,
\end{equation}
where $\varepsilon_i$ denotes noise variable which is independence of $Pa_{X_i}$ and other noise variables, i.e., $ \varepsilon_i \bot \varepsilon_j$, for all $j=1,\dots,|\mathbf{V}|$, where $|\mathbf{V}|$ denotes the number of observed variables in graph, and $f_i$ denotes the functional relationship between cause and effect.

ANM is widely used in statistics, multivariate analysis, and causal discovery \cite{hoyer2008nonlinear,cai2018self,qiao2021causal}. In causal discovery, ANM is capable to alleviate the problem of Markov equivalence class and display an identification ability on causal direction, which may be helpful in the missing data setting. In the paper, we only focus on the class of ANMs that are identifiable in complete data, e.g., the linear non-Gaussian model \cite{shimizu2006linear}, or the nonlinear additive noise model \cite{hoyer2008nonlinear}.

However, without further constraints, 
it is hard to answer the identification problem in the presence of missing data even in the ANM model. Our theoretical results build upon the basic framework of \citet{tu2019causal}, following Assumption 1 $\sim$ 3 in \citet{tu2019causal}, which restrict (i) the missingness indicators are not causes; (ii) conditional independence relations in the observed data also hold in the complete data; (iii) no causal interactions between missingness indicators. Besides, we allow the existence of self-masking missingness (as shown in the following Assumption).

\begin{assumption}[Weak self-masking missingness]\label{theorem:assumption4}
	Weak self-masking missingness refers to missingness in a variable that is caused only by itself, i.e., for any self-masking missing variable $V_i \to R_i$, $V_i\in \mathbf{V}_m$, the parent set of $R_i$ only contains itself $Pa_{R_i}=\{V_i\}$.
\end{assumption}

The key difference to existing methods, such as \citet{tu2019causal, gao2022missdag}, is that we relax the `no self-masking missingness' assumption to a sensible situation (as shown in Assumption \ref{theorem:assumption4}). Moreover, to ensure the asymptotic correctness of the identification algorithm in the presence of weak self-masking missingness, apart from the assumptions of causal Markov, faithfulness, and causal sufficiency, we further require a mild structural assumption.

\begin{assumption}[Structural condition]\label{theorem:assumption5}
	For each indicator of a weak self-masking variable $Z_i \in \mathbf{Z}$, there exists $X,Y\in \mathbf{V}\setminus \mathbf{Z}$ such that $X \Vbar_{d} Y | \mathbf{Z}$ in the ground truth.
\end{assumption}

Note that this assumption generally holds as the real-world structure is generally sparse and such d-separation should often occur. Overall, our \textbf{goal} is to identify the causal structure from missing data under more general assumptions that the self-masking missingness is allowed.

\section{Identifiability Results}
In this section, we address the problem of structure learning using ANM under the proper assumptions discussed above. To leverage the identification results of ANM (refer to Theorem \ref{thm:missing_anm_condition}), it is necessary to answer the identification of missing indicators in advance. We will show that the missing indicator can be found in a constraint-based framework that learns the causal skeleton and missing indicators correctly (Theorem \ref{thm: identification_indicator}). Based on this result, we complete the identification of causal direction and show that the causal structure is identified up to an IN-equivalent pattern for ANM (Definition \ref{def:in_pattern}) and further provide an insightful orientation rule for further searching the equivalent class.

\subsection{Identification of Missing Indicators}

We start with the identification of missing indicators, used as a building block for the ANM identification. Our results are based on a constraint-based framework in which we represent the learned structure as a causal skeleton that entails a set of d-separation relations. In other words, the causal skeleton characters the Markov equivalence class and ignores the causal direction information (e.g., V-structure). In this section, we mainly focus on the identification of missing indicators, together with the causal skeleton, which are mutually complementary.

To complete the identification of missing indicators, we are required to deal with two challenges: (i) CI test in missing data where weak self-missingness is allowed; (ii) identify the causal relations between variable $\mathbf{V}$ and missing indicator $\mathbf{R}$. The result of the first challenge ensures the correctness of the recovered causal skeleton while the latter finds the structural position of missing indicators in the causal skeleton (e.g., the parent set for each missing indicator).

\subsubsection{CI test in missing data in the presence of weak self-masking missingness.}

Consider the first challenge that the goal is to perform the correct CI tests in missing data and then use them to recover the causal skeleton. Under the `no self-masking missingness' assumption, a recent approach by \citet{tu2019causal} has shown that although there are some special cases, in most cases, the CI relation can be directly tested from data by the Test-wise Deletion based method (TD-based method). 
Our first finding is extending the CI relations in test-wise deletion data \cite{tu2019causal} to the case of weak self-masking missingness.

\begin{theorem}\label{thm:self_ind}
	With Assumption 1 $\sim$ 3 in \cite{tu2019causal} and the assumption of weak self-masking missingness, for any ${\displaystyle X,Y\in \mathbf{V}}$, ${\displaystyle \mathbf{Z} \subseteq \mathbf{V} \backslash \{X,Y\}}$, and their corresponding missingness indicators $\displaystyle \mathbf{R}_{X,Y,\mathbf{Z}}$, the CI test between ${\displaystyle X,Y}$ given $\mathbf{Z}$ is always consistent with that without the self-masking missingness, i.e., ${\displaystyle X\Vbar Y|\{\mathbf{Z} ,\mathbf{R}_{X,Y,\mathbf{Z}} \})\Leftrightarrow X\Vbar Y|\{\mathbf{Z} ,\mathbf{R}_{X,Y,\mathbf{Z}} \setminus \mathbf{R}_{S} \})}$ and ${\displaystyle X\not{\Vbar} Y|\{\mathbf{Z} ,\mathbf{R}_{X,Y,\mathbf{Z}} \})\Leftrightarrow X\not{\Vbar} Y|\{\mathbf{Z} ,\mathbf{R}_{X,Y,\mathbf{Z}} \setminus \mathbf{R}_{S} \})}$ where $\mathbf{R}_{S} \!=\!\left\{R_{i} |V_{i}\rightarrow R_{i}\right\}$ is the set of weak self-masking indicators.
\end{theorem}

Here, the `consistent' means that the CI relations that contain the weak self-masking missingness are equal to the CI relations after the elimination of such self-masking missingness. Based on Theorem \ref{thm:self_ind}, the existing results for CI relations in missing data can be naturally extended to the case with the existence of weak self-masking. Thus, we can further extend the result from (\citet{tu2019causal}, Prop. 2) that the CI relations in the observed MNAR data may be different from those in the complete data. Such phenomenon, as shown in Corollary \ref{coro:ci_self}, is due to conditioning on a missingness indicator which is a collider, and thus introduces a spurious dependence.

\begin{corollary}\label{coro:ci_self}
	Suppose that $X$ and $Y$ are not adjacent in the true causal graph and that for any variable set $\mathbf{Z} \subseteq \mathbf{V}\setminus \{X, Y\}$ such that $X \Vbar Y |\mathbf{Z}$. Then under Assumption 1 $\sim$ 3 in \cite{tu2019causal} and the assumption of weak self-masking missingness, $X \not\Vbar Y |\{\mathbf{Z}, R_X, R_Y, \mathbf{R_Z}\}$ if and only if for at least one variable $V \in \{X\} \cup \{Y\} \cup \{\mathbf{Z}\}$, such that $X$, $Y$ are the direct parents or direct ancestor of $R_V$.%its missingness indicator is neither the direct common child nor a descendant of the direct common child of $X$ and $Y$.
\end{corollary}

Such spurious dependence is supposed to be addressed using methods like Density Ratio Weighted correction (DRW) \cite{mohan2013graphical} to recover the original distribution. However, these methods only hold in the assumption of no self-masking missingness and are not feasible when self-masking missingness exists. Interestingly, we find that the distribution of missing data can still be recovered up to a reasonable conditional distribution that only conditions on the self-masking variables, as given in Prop. \ref{lemma:recover}.

\begin{proposition}\label{lemma:recover}
	With assumptions 1 $\sim $ 3 in \cite{tu2019causal} and assumption of weak self-masking missingness, given a m-graph $\displaystyle G$, the joint distribution ${\displaystyle P(V)}$ is recoverable up to ${\displaystyle P(V|\mathbf{R}_{S} )}$, where ${\displaystyle \mathbf{R}_{S} =\left\{R_{i} |V_{i}\rightarrow R_{i}\right\}}$ is the collection of the self-masking missingness indicators. Then, we have
	\begin{equation}
		P(V|\mathbf{R_{S}} )=\frac{P(\mathbf{R}_{V\backslash S} =0,V{\displaystyle |\mathbf{R}_{S}})}{\!\!\!\prod\limits_{i\in \{i|R_{i} \in \mathbf{R}_{V\backslash S}\}}\!\!\!\!\!\!\!\!\!\!\! P\left( R_{i} =0|Pa^{o}_{R_{i}} ,Pa^{m}_{R_{i}} ,\mathbf{R}_{Pa^{m}_{R_{i}}}\right)}
	\end{equation}
	where $\displaystyle Pa^{o}_{R_{i}} \subseteq V_{0}$ and $\displaystyle Pa^{m}_{R_{i}} \subseteq V_{m}$ denote the parents of $\displaystyle R_{i}$, $\displaystyle \mathbf{R}_{V\backslash S}$ is the non self-masking missingness indicators.
\end{proposition}

Based on Prop. \ref{lemma:recover} and Theorem \ref{thm:self_ind}, one may correct the CI relations in the recovered distribution and hence ensure the causal skeleton is reconstructed correctly by CI test and correcting method. The implementation will be provided in the algorithm section. By this, the first challenge is solved.

\subsubsection{Identifying missing indicators}
Now, we deal with the second challenge: how to identify the missing indicator variables in the skeleton. Our solution is built upon the constraint-based method that the parent set of missing indicators is identified through CI testing. 

One straightforward way is to perform test-wise deletion CI test between each missingness indicator and the variables. As discussed by \citet{tu2019causal}, if there does not exist self-masking missingness, and the assumption 1 $\sim$ 3 in \cite{tu2019causal} hold, then all missingness indicators can be identified correctly since CI relations in observed data will be consistent with that in complete data, i.e., $R_{V_i}\Vbar V_j | \mathbf{Z} \Leftrightarrow  R_{V_i}\Vbar V_j^* | \mathbf{Z}$ for any $\mathbf{Z}\subseteq \mathbf{V}\setminus V_j$ where $V_j^*$ is the observed variable in missing data. 

However, such a result only holds for the case without self-masking missingness. When a self-masking missingness $V_i\to R_{V_i}$ exist, we have $R_{V_i}\not{\Vbar} V_i$ but $R_{V_i}\Vbar V_i^*$. The reason is that we can only observe the value of $V_i$ when $R_{V_i}=1$ and missing when $R_{V_i}=0$ such that $R_{V_i}$ is a constant in testing leading to independence.
Such an untestable problem will introduce spurious dependence on weak self-masking missingness indicators. An example is given in the following.

\begin{example}[Untestable CI relations in self-masking indicator]\label{example:untestable CI}
	As shown in Fig. \ref{fig:illust_simple_example}, since $Y$ being missing, the conditional independence $\{X, W\}\bot R_Y|Y$ is untestable. It will result in the relations between $X, W$, and $R_Y$ being dependent, and thus $X, W$ is incorrectly identified as the parent of $R_Y$.
\end{example}

To tackle this problem, fortunately, we find that such incorrect results could be corrected and the indicators of weak self-masking variables are identifiable by detecting the conflict of CI relations under a mild structure condition (Assumption \ref{theorem:assumption5})\textemdash there exist two variables that are d-separated by the weak self-masking missingness variable.

\begin{lemma}[Identification of self-masking indicator]\label{lemma: self_indicator}
	Suppose Assumption 1 $\sim$ 3 in \cite{tu2019causal} holds, and further assume weak self-masking missingness and structure condition hold. A variable $Z_i \in \mathbf{Z}$ is a self-masking missingness variable if there exits $X,Y\in \mathbf{V}$ such that ${X,Y}\not \Vbar R_{Z_i}|\{R_X,R_Y\}$, and (i) a simply test-wise deletion CI test yields $X \Vbar Y |\{\mathbf{Z},R_X,R_Y,\mathbf{R}_{\mathbf{Z}}\}$; or (ii) after the correction by Prop. \ref{lemma:recover}, the CI test yields $X \Vbar Y |\{\mathbf{Z},\mathbf{R}_S\}$, where $\mathbf{R}_S$ is test-wise self-masking missingness indicators.
\end{lemma}

Roughly speaking, Lemma \ref{lemma: self_indicator} provides a correcting method for identifying an indicator of self-masking missingness. An illustrative example is provided in the following.

\begin{example}
	According to Example \ref{example:untestable CI}, in Fig. \ref{fig:illust_simple_example}(a), $X,W$ will incorrectly be identified as a parent of $R_Y$. That is, we have ${X,W}\not \Vbar R_{Y}|\{R_X,R_W\}$. Moreover, we can further find that $X \Vbar W |\{Y, Z, R_Y, R_W\}$ holds which is unreasonable because $X \bot W|Y, R_Y$ is not possible to hold if $R_Y$ is the collider for $X$ and $W$ under faithfulness assumption. Therefore, $R_Y$ must be the weak self-masking missingness indicator and we obtain Lemma \ref{lemma: self_indicator}.
\end{example}

Based on the result in Lemma \ref{lemma: self_indicator}, the identification of missing indicators is well addressed.
\begin{theorem}[Identification of missing indicators]\label{thm: identification_indicator}
	Suppose Assumption 1 $\sim$ 3 in \cite{tu2019causal} holds, and further assume weak self-masking missingness and structure condition hold. The causal relations of missing indicators are identifiable.
\end{theorem}

By this, one may learn the causal skeleton and the causal relations of missing indicators correctly based on the above theoretical results. A practical identification algorithm (\textit{SM-MVPC}) is presented in the `algorithm' section. Now, we are capable to cope with the next issue\textemdash identification of causal direction.

\subsection{Identification of Causal Direction}

In this section, we aim to address the identification of causal directions in the learned causal skeleton by ANM. One of the straightforward ways to identify the causal direction is to recover all relevant distributions that are required for ANM identifiability. However, the recovery procedure is generally harder and could be inaccurate. Thus, in this work, instead of correcting all distributions, we aim to investigate the identification of causal direction using only the independence noise property in ANM based on the learned causal skeleton.

To do so, we first discuss the identifiable ANM under the complete data and then investigate the identifiability of ANM under the missing data, resulting in a sufficient and necessary condition for that identifiability (Theorem \ref{thm:missing_anm_condition}). To further leverage the identifiability results from ANM, we formulate the IN-equivalent pattern to characterize the independent noise property (Definition \ref{def:in_pattern}), which allows us to orient the causal direction whenever possible by an insightful orientation rule (Theorem \ref{Orientation Rule}).

The key to solving the identification of causal direction is the independent noise property of ANM. Below, we will show how the causal direction is identified by ANM in the complete data, which defines the concept of `\textit{identifiable}'.

\begin{remark}
	Given an ANM $V_i=f_i(Pa_{V_i})+\varepsilon_i$ in complete data, one can identify the causal direction by testing the independence between the residuals of regression and the hypothesis cause, such that the independence holds only in the correct causal direction, e.g., $Pa_{V_i} \Vbar V_i - f_i(Pa_{V_i})$.
\end{remark}

However, when data contain missing values, the independence between residual and cause variables may no longer holds. For example, as shown in Fig. \ref{fig:illust_simple_example}, the ANM of $X \rightarrow Y$ is not identifiable because
$Y^* -f_i(X) \not\bot X$ in missing data. Thus, it is necessary to understand in what cases the ANM is identifiable or not in the presence of missing data. In the following theorem, we provide a sufficient and necessary condition for the identifiability of an ANM causal pair under the missing data.

\begin{theorem}[Identifiability of ANM in missing data]\label{thm:missing_anm_condition}
	Given an m-graph $G$, an additive noise model
	\begin{equation*}
		V_i=f_i(Pa_{V_i})+\varepsilon_i,\quad Pa_{X_i}\Vbar \varepsilon_i
	\end{equation*}
	is identifiable in complete data but not identifiable in missing data if and only if there exists a directed path in $G$ that starts from one of the parent $V_j \in Pa_{V_i}$ and ends at missing indicator $R_{V_j}$ or $R_{V_i}$.
\end{theorem}

\begin{example} [An example for illustrating Theorem \ref{thm:missing_anm_condition}]
	As shown in Fig. \ref{fig:illust_simple_example}, one can see that the ANM in the causal pair between $Y, Z$ and $W$ is identifiable even though there exist self-masking variables. However, the causal pair in between $X, Y$ is non-identifiable due to there exists a directed path from $X$ to $R_Y$.
\end{example}

Theorem \ref{thm:missing_anm_condition} has a graphical implication: in the missing data, the identification of ANM is determined by the structural position of the missing indicator in the causal skeleton. To further characterize the identification of ANM in a causal skeleton, we devise the \textbf{I}ndependent \textbf{N}oise-equivalent (IN-equivalent) pattern using a partially directed graph:

\begin{definition}[IN-equivalent pattern]\label{def:in_pattern}
	An IN-equivalent pattern is a partially directed m-graph, which has the identical adjacency as the original m-graph and which has oriented edge $Pa_{V_{i}}\rightarrow V_{i}$ if and only if $Pa_{V_{i}} \Vbar V_{i} -f_{i}( Pa_{V_{i}}) |\{R_{V_{i}} ,\mathbf{R}_{Pa_{V_{i}}} \}$ where $f_{i}$ is a regression of $\displaystyle V_{i}$ on $Pa_{V_{i}}$.
\end{definition}

In other words, the IN-equivalent pattern encodes the set of independent noise of ANM, where the direct edge represents the identifiable ANM in this edge while the bi-directed edge represents that there is no distinguishable independent noise of ANM for identifying such causal direction.

Similar to Meek's rule \cite{meek1995causal} in the Markov equivalence pattern, given the IN-equivalent pattern, we are able to further orient the undetermined edges by making no identifiable ANM structure and no cycle in the pattern. Our idea is derived from the fact that if (i) an undirected edge $X - Y$ in the IN-equivalent pattern is not identified as $X \rightarrow Y$ (causal direction) due to Theorem \ref{thm:missing_anm_condition}, and (ii) ANM reject the reverse direction $X \leftarrow Y$ by the violation of independent noise, then we can infer $X \rightarrow Y$ uniquely because unidentifiable ANM for $X \rightarrow Y$ is only caused by (i). To characterize such fact and further consider the no cycle constraint, we are able to graphically characterize the identifiability of ANM in missing data using the following potential non-identifiable paths:

\begin{definition}[Potential Non-identifiable Paths]
	For a node $V_i \in \mathbf{V}$, the potential non-identifiable paths $\mathbf{P}_i$ is a set of paths w.r.t. $V_i$ on the missingness graph $G(\mathbb{V},\mathbf{E})$ such that for each path in $\mathbf{P}_i$ satisfies: 1) each edge in the path is either $\rightarrow $ or $-$; 2) the path starts from $V_{j} \in \{V_{j} |V_{j} - V_i \in \mathbf{E}\}$ which is one of the undirected neighbors of $V_i$, and the corresponding end point is $R_{V_i}$ or $R_{V_j}$; 3) for every undirected edge $V_k - V_l$ in the path it can not create cycle in the graph if we orient $V_k\rightarrow V_l$, i.e., there does not exist a directed path from $V_l$ to $V_k$.
\end{definition}

Based on the potential non-identifiable paths, we conclude the following orientation rule:

\begin{theorem}[Orientation Rule]\label{Orientation Rule}
	Given an IN-equivalent pattern represented by a partially directed m-graph $G(\mathbb{V},\mathbf{E} )$, where edges $\mathbf{E}$ may contain directed edges and undirected edges. If the potential non-identifiable paths of $V_i$ is empty, then we orient every undirected neighbor $V_j \in \{V_j|V_j - V_i\}$ as $V_i\to V_j$.
\end{theorem}

\begin{example} [Illustrating for orientation rule]
	Take Fig. \ref{fig:anm_selfmasking_example}(c) as an example, for node $Y$, there exists one potential non-identifiable path $X-Y\to R_X$. Meanwhile, for $X$ there does not exist such type of path. We infer  $X\to Y$ by Theorem \ref{Orientation Rule}.
\end{example}

Intuitively, since all IN-equivalent graphs share the same independence noise, if there exists an orientation that introduces a new independence noise, then such an orientation is invalid as it will break the equivalent pattern. Thus, if we can find a unique structure that makes no identifiable ANM structure in missing data, then we can orient the edge whenever possible.

\section{The algorithms for Learning Causal Structure from Missing Data}

In this section, we provide the implementation of identifiability results, consisting of two practical algorithms for learning the causal skeleton and causal direction, respectively.

\begin{algorithm}[h]
	\caption{Self-Masking MVPC (Simplified)}
	
	% \KwIn{this text}
	% \KwOut{how to write algorithm with \LaTeX2e }
	\tcp{{\color{RoyalBlue} \text{\textrm{\footnotesize Learning skeleton by PC algorithm in deleted data} }}}
	$G \leftarrow$ Test-wise deletion PC \cite{spirtes2000causation}; \\
	\tcp{{\color{RoyalBlue} \text{\textrm{\footnotesize Learning missing indicators in skeleton } }}}
	Detect direct causes of missingness indicators by testing CI relations between each missingness indicator and the variables in $\mathbf{V}$; \\
	Find the weak self-masking indicators by Lemma \ref{lemma: self_indicator}; \\
	\tcp{{\color{RoyalBlue} \text{\textrm{\footnotesize Reconstructing the causal skeleton by a correcting strategy } }}}
	Detecting potential extraneous edges and recovering the true causal skeleton in an iteration way; \\
	\label{alg:self-mvpc}
\end{algorithm}

As shown in Algorithm 1 (named as Self-Masking MVPC, abbreviated as SM-MVPC), it is an extension of MVPC \cite{tu2019causal} for learning causal skeleton allowing the presence of self-masking missingness. Here, due to space limitation, we only present a simplified version of the full algorithm, and the complete algorithm is given in Appendix. In SM-MVPC, we first apply the classic PC algorithm in deleted data to obtain an initial skeleton among $\mathbf{V}$ and further to find the missingness indicators based on Theorem \ref{thm: identification_indicator}. Finally, we iteratively remove potential redundant edges in the recovered distribution (Prop. \ref{lemma:recover}) to obtain the true causal skeleton.

Next, we provide an algorithm for learning causal direction based on the output of SM-MVPC, as shown in Algorithm 2 (\textit{LCS-MD}). We first learn the causal direction by enumerating all candidate parent sets of each node and testing the independence between residuals of regression and hypothesis causal variables. The above procedure outputs an IN-equivalent pattern. Next, we orient the undirected edges according to the orientation rules (Theorem \ref{Orientation Rule}). Our algorithm outputs a partially directed acyclic graph with maximum direction information.

\begin{algorithm}[h]
	\caption{Learning Causal Structure from Missing Data (LCS-MD)}
	\KwIn{The causal skeleton $G$ (including all missing indicators) learned by \textit{SM-MVPC} and  the dataset $\mathcal{D}$ with the observed variable set $\mathbf{V}$}
	\KwOut{A partially directed acyclic graph among the observed variables }
	%begin with a comments
	\tcp{{\color{RoyalBlue} \text{\textrm{\footnotesize Estimating IN-Equivalent pattern } }}}
	\For {$V_i \in \mathbf{V}$}{
		Find the maximum parents set from adjacent with $\operatorname{Pa}_{V_i} \subseteq \operatorname{Adj}_{V_i}$;\\
		\If{$V_i-f_i(\operatorname{Pa}_{V_i})\bot \operatorname{Pa}_{V_i}$}{
			Orient $\operatorname{Pa}_{V_i} \rightarrow V_i$ in $G$; \\
		}
		
	}
	\tcp{{\color{RoyalBlue} \text{\textrm{\footnotesize Perform Orientation rule on IN-equivalent pattern } }}}
	
	\While{no undirected edges in $G$ can be oriented; }{
		Orient the undirected edges based on Theorem \ref{Orientation Rule}.
	}
	\textbf{Return}  partially directed acyclic graph $G$
	\label{alg:LCS-MD}
\end{algorithm}

\begin{theorem}[Soundness of LCS-MD]
	Suppose that the data over variables $\mathbf{V}$ was generated by ANM and assumptions 1 $\sim$ 3 in \cite{tu2019causal}, assumptions of weak self-masking missingness and structural condition hold. Let $G$ denote the output of LCS-MD, all directed edge in $G$ is correctly oriented.
\end{theorem}

\section{Experiments}

In this section, we verify the proposed method with baselines through simulation studies and real-world dataset studies. Our baseline methods include MVPC \cite{tu2019causal} and MissDAG \cite{gao2022missdag}. We evaluate the methods in estimating the skeleton and inferring causal directions, respectively.

\subsection{Synthetic Experiments}

In synthetic experiments, we conducted two different control experiments: (1) the sensitivity of sample size, and (2) the sensitivity of the number of weak self-masking missingness for testing the performance of causal skeleton learning (including the missingness indicators) and causal structure learning, respectively. Each experiment is conducted on MNAR with the existence of weak-self masking missingness. The sensitivity experiments are controlled by traversing the controlled parameter while keeping other setting fixed as default. The ranging of sample size, and the number of weak self-masking missingness: $\{3000,5000,\mathbf{7000},9000,15000\}$,  $\{1,2,\mathbf{3},4\}$, respectively. The default setting is marked as bold. Each experiment is conducted 100 times with different seeds. The causal graph has 10 variables with at least 3 missing variables.

Our data generation process follows the constraint of the assumptions in this work and the missing data are generated according to the missingness indicators following \citet{liu2021greedy}. In detail, for each causal pair, $f_i$ is constructed from a two-layer random MLP with 50 hidden layers $V_i=f(Pa_{V_i})+\varepsilon_{V_i}$, where the noise $\varepsilon_{V_i}\sim \mathcal{U}(-1,1)$. 

In all experiments, Structural Hamming Distance (SHD), Precision, Recall, and F1 are used as the evaluation metrics. For the independence test, the Hilbert-Schmidt Independence Criterion (HSIC) \cite{DBLP:journals/jmlr/GrettonHSBS05,zhang2018large} test is used, and we set the significance level as $\alpha = 0.01$. Due to the space limitation, only the F1 metric is provided, other metrics are provided in the supplementary material.

\paragraph{Results in Causal Skeleton Identification.}
Results are tested with the skeleton of the ground truth. As shown in Fig. \ref{fig:sensitivity}, one can see that SM-MVPC achieves the best performance in all cases since existing methods are not designed for the MNAR with self-masking cases, which shows the effectiveness of our proposal in learning causal skeletons. 

For the sensitivity of the number of weak self-masking variables given in Fig. \ref{fig:sensitivity}(b), our method is not sensitive while the performance of MVPC decreases as the number of weak self-masking variables increases. It suggests the correctness of our method and the existence of redundant edges in MVPC due to the self-masking variables. For other methods, based on Theorem \ref{thm:self_ind} it is reasonable to be not sensitive. In addition, in Fig. \ref{fig:sensitivity}, the average accuracy of the identification of self-masking increases as the sample size grows which also verifies the effectiveness in identifying the missing indicator.

\begin{figure*}[t]
	\centering
	\subfigure[Sensitivity of sample size]{
		\includegraphics[height=0.22\textwidth,width=0.3\textwidth]{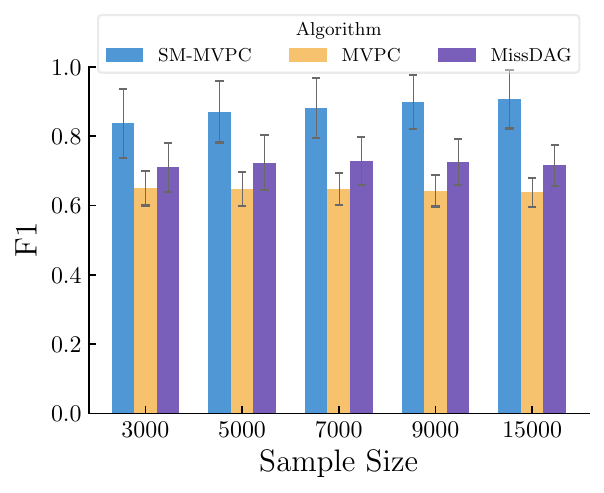}
		\label{fig:ske:Sample_SizelinearSMNARF1}
	}
	\subfigure[Sensitivity of the weak self-masking]{
		\includegraphics[height=0.22\textwidth,width=0.3\textwidth]{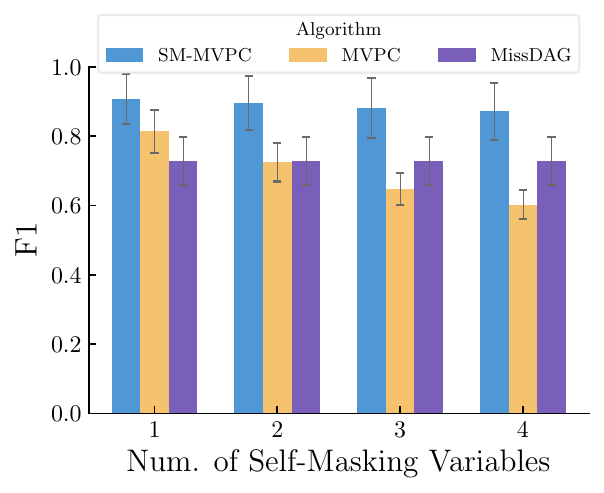}
		\label{fig:ske:Num_of_self_nodelinearSMNARF1}
	}
	\subfigure[Accuracy of self-masking indicator]{
		\includegraphics[height=0.22\textwidth,width=0.3\textwidth]{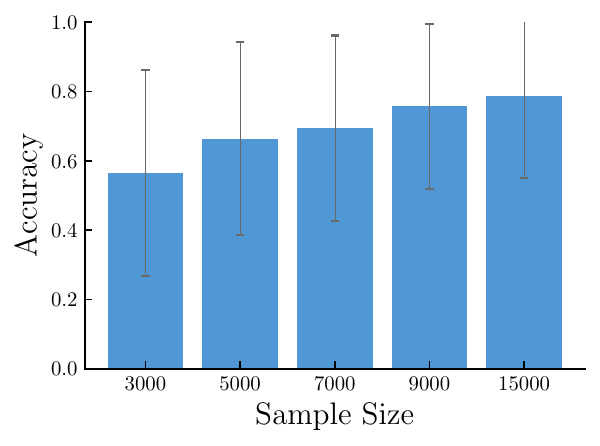}
		\label{fig:ske:IndegreelinearSMNARF1}
	}
	\caption{Experiments for Skeleton Learning}	
	\label{fig:sensitivity}
\end{figure*}

\begin{figure*}[t]
	\centering
	\subfigure[Sensitivity of sample size]{
		\includegraphics[height=0.22\textwidth,width=0.3\textwidth]{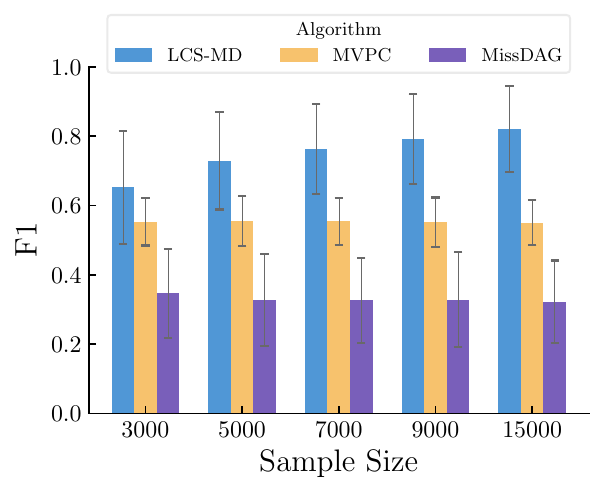}
		\label{fig:ori:Sample_SizelinearSMNARF1}
	}
	\subfigure[Sensitivity of the weak self-masking]{
		\includegraphics[height=0.22\textwidth,width=0.3\textwidth]{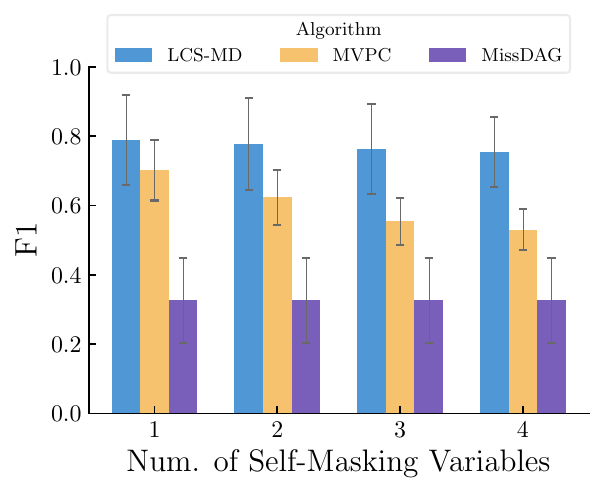}
		\label{fig:ori:Num_of_self_nodelinearSMNARF1}
	}
	\subfigure[Applying rule with correct skeleton]{
		\includegraphics[height=0.22\textwidth,width=0.3\textwidth]{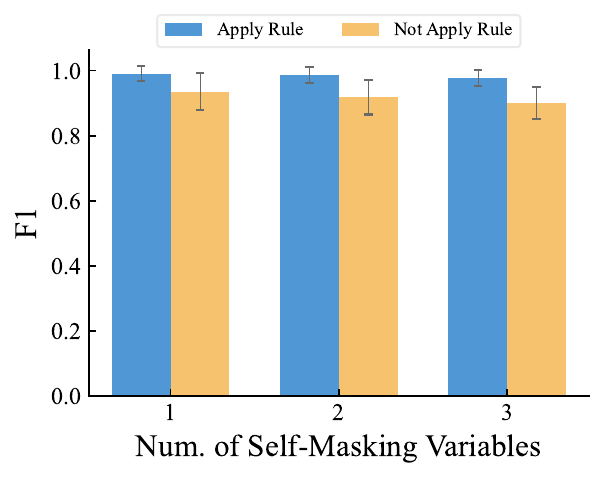}
		\label{fig:ori:IndegreelinearSMNARF1}
	}
	\caption{Experiments for Structure Learning}	
	\label{fig:sensitivity_IN_equivalent}
\end{figure*}

\paragraph{Results in Causal Direction Identification.}
For learning causal structure, as shown in Fig. \ref{fig:sensitivity_IN_equivalent}, our method outperforms all the baseline methods. Moreover, compared with the experiments in the skeleton, the performance gap with baseline methods is even larger, which is thanks to the identifiability of ANM. As for the sensitivity, all methods are sensitive to the sample size. While our method is not sensitive to the weak self-masking variables but MVPC is sensitive. The above results are consistent with the experiments on the skeleton. Moreover, to verify the effectiveness of the orientation rule, we apply the rule on a correct IN-equivalent pattern and the results in Fig. \ref{fig:sensitivity_IN_equivalent} show promising as it successfully identified most of the directions that are not identifiable by ANM.

\subsection{Real World Experiments}

In this section, we applied our method to a real-world dataset, Cognition and Aging in the Chronic Fatigue Syndrome (CFS), provided by \cite{heins2013process}. The dataset contains six variables with 183 records: fatigue, focusing on symptoms (focusing), sense of control over fatigue (control), objective activity (oActivity), physical activity (pActivity), and physical functioning (functioning). It is designed to investigate the cause of fatigue. Moreover, this dataset contains a few missing values. The result is given in Fig. \ref{fig:real_world}, and LCS-MD finds that pAactivity, focusing, control and functioning are the direct causes of fatigue which is consistent with the conclusion drawn by \cite{heins2013process,rahmadi2017causality}. In addition, objective activity also has an indirect effect on fatigue by the pActivity variable. This is reasonable because physical activity (pActivity) also belongs to one type of objective activity. Overall, the result verifies the effectiveness of our method.

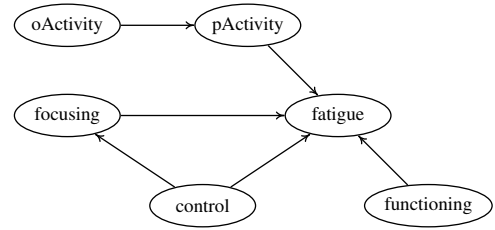
\begin{figure}[t]
	\begin{center}
		\begin{tikzpicture}[scale=1.2, line width=0.5pt, inner sep=0.2mm, shorten >=.1pt, shorten <=.1pt]
			
			\tikzstyle{every node}=[font=\small,scale=0.8,minimum height =20pt,minimum width =50pt]

			\draw (0, 0) node(L1) [ellipse,draw, align=center]{oActivity};

			\draw (2, 0) node(L2) [ellipse,draw, align=center]{pActivity};

			\draw (0, -1) node(L3) [ellipse,draw, align=center]{focusing};

			\draw (1.5, -2) node(L4) [ellipse,draw, align=center]{control};

			\draw (4, -2) node(L6) [ellipse,draw, align=center]{functioning};

			\draw (3, -1) node(L5) [ellipse,draw, align=center]{fatigue};

			\draw[-arcsq] (L1) -- (L2) node[pos=0.5,sloped,above] {};
			
			\draw[-arcsq](L2) -- (L5) node[pos=0.5,sloped,above] {};
			
			\draw [-arcsq](L3) -- (L5) node[pos=0.5,sloped,above] {};
			\draw [-arcsq](L4) -- (L3) node[pos=0.5,sloped,above] {};
			
			\draw [-arcsq](L6) -- (L5) node[pos=0.5,sloped,above] {};
			
			\draw [-arcsq](L4) -- (L5) node[pos=0.5,sloped,above] {};

		\end{tikzpicture}\\
		\caption{Result from LCS-MD in the Chronic Fatigue Syndrom (CFS) study.}  \label{fig:real_world} 
	\end{center}
\end{figure}

\section{Conclusion}
In this work, we studied the causal discovery problem on the observational data with missing values. By taking the advantage of additive noise model, we greatly extend the identifiability results of the causal discovery methods, including that the causal skeleton is identifiable up to the case with weak self-masking missingness and the causal direction is identifiable up to an IN-equivalent pattern. Based on the theoretical results, we propose two algorithms, SM-MVPC and LCS-MD, to discover the causal structure on the data with missing values. The proposed theorems and the algorithms take a meaningful step in understanding the missingness mechanism. How to relax the assumptions and how to improve efficiency would be interesting future directions.

\section{Acknowledgement}
This research was supported in part by National Key R\&D Program of China (2021ZD0111501), National Science Fund for Excellent Young Scholars (62122022), Natural Science Foundation of China (61876043, 61976052), the major key project of PCL (PCL2021A12). ZM's research was supported by the China Scholarship Council (CSC).

\bibliography{aaai24}
\appendix
\clearpage
\newtheorem{innercustomthm}{Theorem}
\newenvironment{customthm}[1]
{\renewcommand\theinnercustomthm{#1}\innercustomthm}
{\endinnercustomthm}

\newtheorem{innercustomcor}{Corollary}
\newenvironment{customcor}[1]
{\renewcommand\theinnercustomcor{#1}\innercustomcor}
{\endinnercustomcor}

\newtheorem{innercustomlem}{Lemma}
\newenvironment{customlem}[1]
{\renewcommand\theinnercustomlem{#1}\innercustomlem}
{\endinnercustomlem}

\newtheorem{innercustomprop}{Proposition}
\newenvironment{customprop}[1]
{\renewcommand\theinnercustomprop{#1}\innercustomprop}
{\endinnercustomprop}

\newtheorem{innercustomremark}{Remark}
\newenvironment{customremark}[1]
{\renewcommand\theinnercustomremark{#1}\innercustomremark}
{\endinnercustomremark}

\numberwithin{equation}{section}
\setcounter{theorem}{0}
\setcounter{lemma}{0}
\setcounter{remark}{0}
\setcounter{proposition}{0}
\setcounter{corollary}{0}
\setcounter{section}{0}
\setcounter{secnumdepth}{1}
\section{Notation and Terminology}

\begin{table}[h]
	\centering
	\small
	\begin{tabular}{c|c}
		\toprule
		Symbols & Definitions and Descriptions \\ 
		\midrule
		$Pa_{V_i}$       & parent of $V_i$, i.e., $Pa_{V_i}=\{V_j|V_j \to V_i\}$                     \\
		$Ch_{V_i}$       & children of $V_i$, i.e., $Ch_{V_i}=\{V_j|V_i\to V_j\}$                 \\
		$Anc_{V_i}$       & ancestors of $V_i$, i.e., $Anc_{V_i}=\{V_j|V_j \rightsquigarrow V_i\}$       \\
		$Des_{V_i}$       & descendants of $V_i$, i.e., $Des_{V_i}=\{V_j|V_i \rightsquigarrow V_j\}$       \\
		$Adj_{V_i}$       & adjacent of $V_i$, i.e., $Adj_{V_i}=\{V_j|V_j - V_i\}$    \\
		$R_{V_i}$               &missingness indicator of missing variable $V_i$\\
		$P(V)$      & the distribution of $V$\\
		$CI(V_i, V_j|\mathbf{V_k})$          &  CI relation between $V_i$ and $V_j$ given $\mathbf{V_k}$\\
		\bottomrule
	\end{tabular}
	\caption{Table of Notations}
\end{table}

Below, we provide some graphical notation used in our work, which is mainly derived from the \cite{pearl2009causality,spirtes2000causation}.

\begin{definition}[Path and Directed Path]
	In a DAG, a \textbf{path} $P$ is a sequence of nodes $(V_1,...V_r)$ such that $V_i$ and $V_{i+1}$ are adjacent in $\mathcal{G}$, where $1 \le i<r$. Further, we say a path $P=( V_{i_0} ,V_{i_1} ,\dots ,V_{i_{k}})$ in $G$ is a \textbf{directed path} if it is a sequence of nodes of $G$ where there is a directed edge from $V_{i_j}$ to $V_{i_{(j+1)}}$ for any $0\leq j\leq k-1$.
\end{definition}

\begin{definition}[Collider]
	A \textbf{collider} on a path $\{V_1,...V_p\}$ is a node $V_i$ , $1< i < p$, such that $V_{i-1}$ and $V_{i+1}$ are parents of $V_{i}$.
\end{definition}

Graphically, we also say a collider is a `V-structure'.

\begin{definition}[d-separation]
	A path $p$ is said to be d-separated (or blocked) by a set of nodes $\mathbf{Z}$ if and only if
	\begin{itemize}
		\item $p$ contains a chain $V_i \to V_k \to V_j$ or a fork $V_i \leftarrow V_k \rightarrow V_j$ such that the middle node $V_k$ is in $\mathbf{Z}$, or 
		\item $p$ contains a collider $V_i \to V_k \leftarrow V_j$ such that the middle node $V_k$ is not in $\mathbf{Z}$ and such that no descendant of $V_k$ is in $\mathbf{Z}$.
	\end{itemize}
\end{definition}
A set $\mathbf{Z}$ is said to d-separate $\mathbf{A}$ and $\mathbf{B}$ if and only if $\mathbf{Z}$ blocks every path from a node in $\mathbf{A}$ to a node in $\mathbf{B}$. We also denote as $\mathbf{A} \bot \mathbf{B} |\mathbf{Z}$ in the causal graph model.

\begin{definition}[Causal skeleton]\label{def:skeleton}
	A undirected graph $S = (\mathbf{V}_{S}, \mathbf{E}_{S})$ represents the skeleton of a causal DAG $\mathcal{G}=(\mathbf{V}_{\mathcal{G}},\mathbf{E}_{\mathcal{G}})$ if
	\begin{equation}
		(X \rightarrow Y) \in \mathbf{E}_{\mathcal{G}} \vee (Y \rightarrow X) \in \mathbf{E}_{\mathcal{G}}	\Longleftrightarrow (X - Y) \in \mathbf{E}_{S}
	\end{equation}
\end{definition}

Here, we further provide the details to the assumption 1 $\sim$ 3 by \cite{tu2019causal}, which also is required in our paper (we give a brief description in the `\textbf{Problem Definition}' of the main paper due to the space limitation).

\begin{assumption}[Missingness indicators are not causes]\label{theorem:assumption1}
	No missingness indicator can be a cause of any variable.
\end{assumption}

\begin{assumption}[Conditional independence relations in the observed data also hold in the complete data]\label{theorem:assumption2}
	Any conditional independence relation in the observed data
	also holds in the unobserved data; formally, 
	$X\bot Y |\{\mathbf{Z},\mathbf{R_K}=0\} \Leftrightarrow X\bot Y |\{\mathbf{Z},\mathbf{R_K}=1\}$.
\end{assumption}

\begin{assumption}[No causal interactions between missingness indicators]\label{theorem:assumption3}
	No missingness indicator can be a deterministic function of any other missingness indicator.
\end{assumption}

These assumptions are widely used in dealing with missing data, such as \cite{mohan2013graphical,tu2019causal,mohan2018estimation,little2019statistical,gao2022missdag}.

Furthermore, we provide the existing theorem used in our work. That is, the notation of `identifiable'.

\begin{definition}[Identifiable ANM in complete data] \label{def:identifiable}
	An additive noise model $V_i=f_i(Pa_{V_i})+\varepsilon_i$ is identifiable if the joint distribution $P(V_i,Pa_{V_i})$ admits the additive noise model in the causal direction such that $Pa_{V_i} \Vbar \varepsilon_i$, i.e., $P(V_i,Pa_{V_i})=P(V_i-f_i(Pa_{V_i}))P(Pa_{V_i})$, but not in the reverse direction.
\end{definition}

\section{Related work}

In this section, we briefly review the recent progress for learning causal structure from missing data
and traditional methods for causal discovery.

\paragraph{Missing data.}
In contrast to the traditional imputation-based method in processing missing data, the graphical model has demonstrated its expressiveness and usefulness in handling the complex missingness mechanism \cite{mohan2021graphical}. With the graphical model, the missingness mechanism can be formulated by missingness graphs (or m-graph in short) and it can be shown that given a certain m-graph, one can obtain a consistent estimation of the original distribution \cite{mohan2013graphical,NIPS2014_31839b03,ma2021identifiable}. Thus, one straightforward way is based on the recoverability of the missing data by using the Inverse Probability Weight (IPW) to correct every CI test \cite{gain2018structure}. However, it assumes that the missingness causal model is known, which is unrealistic in many real-world applications. \citet{strobl2018fast}, on the other hand, proposes to use a delete-wise independence test and apply FCI algorithm \cite{colombo2012learning} directly, in which the missingness mechanism will be treated as the selection bias. However, it will produce a large number of equivalent classes as it ignores the causal missingness mechanism. Another line of work to handle the missing data is based on Expectation-Maximization (EM) method \cite{rubin1976inference}, which has become a popular technique to conduct causal discovery with missing data, e.g., Structure EM algorithm \cite{friedman1997learning}. Recently, \citet{gao2022missdag} develops a novel EM-based framework for learning causal structure from missing data with additive noise model. However, it requires the underlying missing mechanism is ignorable, which is not applicable to MNAR cases.

\paragraph{Causal discovery.}
Causal discovery from observational data \cite{spirtes2000causation} has find numerous applications in the area like network operation maintenance \cite{cai2022thps,ijcai2023p0633}. There are two typical approaches in the literature including the constraint-based methods and the functional causal model based. The main idea of constraint-based methods is to eliminate the edges that are not consistence to the distribution by testing the conditional independence constraint among the variables. Typical methods include PC algorithm \cite{spirtes2000causation}, and FCI algorithm \cite{colombo2012learning}. One may encode the conditional independence into the score and use the score to search the most plausible DAG according to data, e.g., GES algorithm \cite{chickering2002optimal}. However, the constraint-based methods are not able to distinguish the Markov equivalent class, which may leave some undirected edges that are not uniquely identified. As for the functional causal model based methods, by introducing a properly constrained in the relationship between the effect and its direct cause, the causal direction can be uniquely identified, e.g., Additive Noise Model (ANM) \cite{peters2014causal}, Post-nonlinear model (PNL) \cite{zhang2012identifiability}, linear non-Gaussian model (LiNGAM) \cite{shimizu2006linear}.

\section{Algorithm}

The complete implementation of the Self-Masking MVPC algorithm is given in Algorithm \ref{alg:self-mvpc}, together with some necessary discussion to the correctness of the proposed method.

\begin{algorithm}[h]
	\caption{Self-Masking MVPC (Complete)}
	% \KwIn{this text}
	% \KwOut{how to write algorithm with \LaTeX2e }
	\tcp{{\color{RoyalBlue} \text{\textrm{\footnotesize Learning skeleton by PC algorithm in deleted data} }}}
	Remove edges from complete undirected graph $G$ using PC algorithm \cite{spirtes2000causation} with test-wise deleted data; \\
	\tcp{{\color{RoyalBlue} \text{\textrm{\footnotesize Learning missing indicators in skeleton } }}}
	For each variable $V_i \in \mathbf{V}$ containing missing values and for each $j$ that $j \neq i$, remove the edge from $V_j$ to $R_i$ if they are conditional independence given some subset of $\mathbf{V} \setminus \{V_i, V_j\}$; \\
	\tcp{{\color{RoyalBlue} \text{\textrm{\footnotesize (i) Correcting the result of the previous procedure} }}}
	\tcp{{\color{RoyalBlue} \text{\textrm{\footnotesize (ii) Finding the self-masking missingness indicators} }}}
	For each pair variables $V_i$ and $V_j$ are d-separated by $\tilde{\mathbf{V}}$, where $\tilde{\mathbf{V}} \subseteq \mathbf{V}\setminus \{V_i, V_j\}$, if there exist a missing variable $V_k$ in $\tilde{\mathbf{V}}$ and the missing indicator $R_{V_k}$ is the direct child of $V_i$ and $V_j$, then remove all the parent of $R_{V_k}$ and orient $V_k \rightarrow R_{V_k}$.; \\
	\tcp{{\color{RoyalBlue} \text{\textrm{\footnotesize Reconstructing the causal skeleton by a correcting strategy } }}}
	For each adjacency node $V_i$ and $V_j$, if they have at least one common adjacent variable or missingness indicator, enumerate all possible error indicators and perform the correction method (Proposition \ref{lemma:recover}) for removing the extraneous edges; \\
	Perform correction methods for removing the extraneous edges and update the d-separation set;\\
	Repeat the correcting procedure in Line 3 $\sim$ 5 until no update.\\
\end{algorithm}

Overall, the correctness of \textit{SM-MVPC} is generally guaranteed by theoretical results provided in the `Identification of Missing Indicators' section. Specifically, the causal skeleton can be reconstructed correctly by testing CI relations in the recovered distribution (if necessary) according to Prop. 1. Meanwhile, the missing indicator can be found correctly with the theoretical guarantee by Theorem 2.

\section{Proofs and illustrations}

\subsection{Proof of Theorem 1}
\begin{theorem}
	With Assumption 1 $\sim$ 3 in \cite{tu2019causal} and the assumption of weak self-masking missingness, for any ${\displaystyle X,Y\in \mathbf{V}}$, ${\displaystyle \mathbf{Z} \subseteq \mathbf{V} \backslash \{X,Y\}}$, and their corresponding missingness indicators $\displaystyle \mathbf{R}_{X,Y,\mathbf{Z}}$, the CI test between ${\displaystyle X,Y}$ given $\mathbf{Z}$ is always consistent with that without the self-masking missingness, i.e., ${\displaystyle X\Vbar Y|\{\mathbf{Z} ,\mathbf{R}_{X,Y,\mathbf{Z}} \})\Leftrightarrow X\Vbar Y|\{\mathbf{Z} ,\mathbf{R}_{X,Y,\mathbf{Z}} \setminus \mathbf{R}_{S} \})}$ and ${\displaystyle X\not{\Vbar} Y|\{\mathbf{Z} ,\mathbf{R}_{X,Y,\mathbf{Z}} \})\Leftrightarrow X\not{\Vbar} Y|\{\mathbf{Z} ,\mathbf{R}_{X,Y,\mathbf{Z}} \setminus \mathbf{R}_{S} \})}$ where $\mathbf{R}_{S} \!=\!\left\{R_{i} |V_{i}\rightarrow R_{i}\right\}$ is the set of weak self-masking indicators.
\end{theorem}

\begin{proof}
	We will prove that the CI relations of $CI(X,Y|\{\mathbf{Z} ,{\displaystyle \mathbf{R}_{X,Y,\mathbf{Z}}} \})$ is consistent with CI relations that without the self-masking variables in $X,Y,\mathbf{Z}$, i.e., consistent with $CI(X,Y|\{\mathbf{Z} ,{\displaystyle \mathbf{R}_{X,Y,\mathbf{Z}}} \setminus \mathbf{R}_{S} \})$. That is $X\Vbar Y|\{\mathbf{Z} ,{\displaystyle \mathbf{R}_{X,Y,\mathbf{Z}}} )\Leftrightarrow X\Vbar Y|\{\mathbf{Z} ,{\displaystyle \mathbf{R}_{X,Y,\mathbf{Z}}} \setminus \mathbf{R}_{S} )$ and $X\not{\Vbar } Y|\{\mathbf{Z} ,{\displaystyle \mathbf{R}_{X,Y,\mathbf{Z}}} )\Leftrightarrow X\not{\Vbar } Y|\{\mathbf{Z} ,{\displaystyle \mathbf{R}_{X,Y,\mathbf{Z}}} \setminus \mathbf{R}_{S} )$. To do so, we prove that by contradiction, i.e., (1) $X\not{\Vbar } Y|\{\mathbf{Z} ,{\displaystyle \mathbf{R}_{X,Y,\mathbf{Z}}} \}$ but $X\Vbar Y|\{\mathbf{Z} ,{\displaystyle \mathbf{R}_{X,Y,\mathbf{Z}}} \setminus \mathbf{R}_{S} \}$ and (2) $X\not{\Vbar } Y|\{\mathbf{Z} ,{\displaystyle \mathbf{R}_{X,Y,\mathbf{Z}}} \setminus \mathbf{R}_{S} \}$ but $X\Vbar Y|\{\mathbf{Z} ,{\displaystyle \mathbf{R}_{X,Y,\mathbf{Z}}} \}$.

	We first prove the proposition (1). Due to $X\not{\Vbar } Y|\{\mathbf{Z} ,{\displaystyle \mathbf{R}_{X,Y,\mathbf{Z}}} \}$, there must be an activated path $U$ between $X$ and $Y$ and does not be blocked by $\{\mathbf{Z} ,{\displaystyle \mathbf{R}_{X,Y,\mathbf{Z}}} \}$. On the other hand, we have $X\Vbar Y|\{\mathbf{Z} ,\mathbf{R}_{X,Y,\mathbf{Z}} \setminus \mathbf{R}_{S} \}$ hold, which means that the conditional set $\{\mathbf{Z} ,{\displaystyle \mathbf{R}_{X,Y,\mathbf{Z}}} \setminus \mathbf{R}_{S} \}$ block all activated path between $X$ and $Y$. Therefore, given the $\mathbf{R}_{S}$, at least one path $U$ between $X$ and $Y$ must be activated, which must satisfy either one of the following two cases: (i) $\mathbf{R}_{S}$ is the only vertex on $U$, (ii) $\mathbf{R}_{S}$ construct a collider.
	
	In case (i), it contradicts Assumptions 1 $\sim$ 3 in \cite{tu2019causal} and the assumption of weak self-masking missingness because the weak self-masking indicators are the leaf nodes with only one direct parent in missingness graphs, which can not construct any available path between $X$ and $Y$. For case (ii), there must be a path $U$ containing the collider that is activated by $\mathbf{R}_S$. According to the weak self-masking missingness assumption and the faithfulness assumption, the weak self-masking missingness variable has only one direct parent, which means the $\mathbf{R}_S$ can not construct the collider structure (the collider structure must contain two independent parents node). That contradicts the existence of an activated collider structure, and the original proposition does not hold.

	Now we prove the proposition (2). Since $X \not\Vbar Y|\{\mathbf{Z},\mathbf{R}_{X,Y,\mathbf{Z}}\setminus \mathbf{R}_S\}$, there always
	exists an available path $U$ between X and Y that is not blocked by $\{\mathbf{Z},\mathbf{R}_{X,Y,\mathbf{Z}}\setminus \mathbf{R}_S\}$. According to Assumptions 1 $\sim$ 3 in \cite{tu2019causal} and the assumption of weak self-masking missingness, $\mathbf{R}_S$ is not included in any path between $X$ and $Y$ because $\mathbf{R}_S$ is always the leaf node, which means that there is no path between $X$ and $Y$ that can be blocked by $\mathbf{R}_S$. Thus, the available path $U$ between $X$ and $Y$ is not blocked by $\mathbf{R}_S$ and we have $X \not\Vbar Y|\{\mathbf{Z},\mathbf{R}_{X,Y,\mathbf{Z}}\}$ when $X \not\Vbar Y|\{\mathbf{Z},\mathbf{R}_{X,Y,\mathbf{Z}}\setminus \mathbf{R}_S\}$. That is a contradiction with the original proposition. 
	
	According to the above proofs, we conclude that given $\mathbf{R}_S$, the CI relations are consistent with that without self-masking missingness.
	
\end{proof}

\subsection{Proof of Corollary 1}
\begin{corollary}
	Suppose that $X$ and $Y$ are not adjacent in the true causal graph and that for any variable set $Z \subset \mathbf{V}\setminus \{X, Y\}$ such that $X \Vbar Y |\mathbf{Z}$. Then under Assumption 1 $\sim$ 3 in \cite{tu2019causal} and the assumption of weak self-masking missingness, $X \not\Vbar Y |\{\mathbf{Z}, R_X, R_Y, \mathbf{R_z}\}$ if and only if for at least one variable $V \in \{X\} \cup \{Y\} \cup \{\mathbf{Z}\}$, such that $X$, $Y$ are the direct parents or direct ancestor of $R_V$.%its missingness indicator is neither the direct common child nor a descendant of the direct common child of $X$ and $Y$.
\end{corollary}

\begin{proof}
	This proof is straightforward. According to Theorem 1 and the results provided in \citet{tu2019causal} (Proposition 2), this Corollary hold.
\end{proof}

\subsection{Proof of Proposition 1}
\begin{proposition}
	With assumptions 1 $\sim $ 3 in \cite{tu2019causal} and assumption of weak self-masking missingness, given a m-graph $\displaystyle G$, the joint distribution ${\displaystyle P(V)}$ is recoverable up to ${\displaystyle P(V|\mathbf{R}_{S} )}$, where ${\displaystyle \mathbf{R}_{S} =\left\{R_{i} |V_{i}\rightarrow R_{i}\right\}}$ is the collection of the self-masking missingness indicators. Then, we have
	\begin{equation}
		P(V|\mathbf{R_{S}} )=\frac{P(\mathbf{R}_{V\backslash S} =0,V{\displaystyle |\mathbf{R}_{S}})}{\!\!\!\prod\limits_{i\in \{i|R_{i} \in \mathbf{R}_{V\backslash S}\}}\!\!\!\!\!\!\!\!\!\!\! P\left( R_{i} =0|Pa^{o}_{R_{i}} ,Pa^{m}_{R_{i}} ,\mathbf{R}_{Pa^{m}_{R_{i}}}\right)}
	\end{equation}
	where $\displaystyle Pa^{o}_{R_{i}} \subseteq V_{0}$ and $\displaystyle Pa^{m}_{R_{i}} \subseteq V_{m}$ denote the parents of $\displaystyle R_{i}$, $\displaystyle \mathbf{R}_{V\backslash S}$ is the non self-masking missingness indicators.
\end{proposition}

\begin{proof}
	The observed joint distribution can be decomposed according to $\displaystyle G$:
	\begin{equation}
		\begin{aligned}
			& P(\mathbf{R}_{V\backslash S} =0,V|{\displaystyle \mathbf{R}_{S}})\\
			& =P( V|{\displaystyle \mathbf{R}_{S}}) P(\mathbf{R}_{V\backslash S} =0|V,{\displaystyle \mathbf{R}_{S}})\\
			& =P( V|{\displaystyle \mathbf{R}_{S}})\prod\limits _{i\in \{i|R_{i} \in \mathbf{R}_{V\backslash S}\}} P\left( R_{i} =0|Pa^{o}_{R_{i}} ,Pa^{m}_{R_{i}} ,{\displaystyle \mathbf{R}_{S}}\right) ,
		\end{aligned}
	\end{equation}
	where the second equality is based on the fact that there are no edges between $\displaystyle R$ variables and that there are no latent variables as parents of $\displaystyle R$. Moreover, because $\displaystyle V_{i}$ is not the parent of $\displaystyle R_{i}$, i.e., $\displaystyle V_{i}\not{\in } Pa^{m}_{R_{i}}$, and no edges between $\displaystyle R$ variables and $\displaystyle R$ is not the cause variables, then the parent of $\displaystyle R_{i}$ block every of path from $\displaystyle R_{Pa^{m}_{R_{i}}}$ to $\displaystyle R_{i}$, and we have $\displaystyle R_{i} \bot_d \mathbf{R}_{S} |Pa^{o}_{R_{i}} \cup Pa^{m}_{R_{i}}$ and $\displaystyle R_{i} \bot_d \mathbf{R}_{Pa^{m}_{R_{i}}} |Pa^{o}_{R_{i}} \cup Pa^{m}_{R_{i}}$. This is, for $ R_{i} \in \mathbf{R}_{V\backslash S}$, we have
	\begin{equation}
		\begin{aligned}
			&P\left( R_{i} =0|Pa^{o}_{R_{i}} ,Pa^{m}_{R_{i}} ,{\displaystyle \mathbf{R}_{S}}\right) \\
			&=P\left( R_{i} =0|Pa^{o}_{R_{i}} ,Pa^{m}_{R_{i}} ,\mathbf{R}_{Pa^{m}_{R_{i}}}\right)
		\end{aligned}
	\end{equation}
	Due to the strictly positive of $\displaystyle P( \mathbf{R}=0,\mathbf{V}_{m} ,\mathbf{V}_{o})$, we have $\displaystyle P\left( R_{i} =0|Pa^{o}_{R_{i}} ,Pa^{m}_{R_{i}} ,\mathbf{R}_{Pa^{m}_{R_{i}}}\right)$ are all strictly positive. Thus, the $\displaystyle {\displaystyle P(V|\mathbf{R_{S}} )}$ is identifiablity.
\end{proof}

\subsection{Proof of Lemma 1}
\begin{lemma}[Identification of self-masking indicator]
	Suppose Assumption 1 $\sim$ 3 in \cite{tu2019causal} holds, and further assume weak self-masking missingness and structure condition hold. A variable $Z_i \in \mathbf{Z}$ is a self-masking missingness variable if there exits $X,Y\in \mathbf{V}$ such that ${X,Y}\not \Vbar R_{Z_i}|\{R_X,R_Y\}$, and (i) a simply test-wise deletion CI test yields $X \Vbar Y |\{\mathbf{Z},R_X,R_Y,\mathbf{R}_{\mathbf{Z}}\}$; or (ii) after the correction by Prop. \ref{lemma:recover}, the CI test yields $X \Vbar Y |\{\mathbf{Z},\mathbf{R}_S\}$, where $\mathbf{R}_S$ is test-wise self-masking missingness indicators.
\end{lemma}
\begin{proof}
	We prove by contradiction. Suppose $Z_i$ is not a self-masking missingness variable, but there exits $X,Y\in \mathbf{V}$ such that ${X,Y}\not \Vbar R_{Z_i}|\{R_X,R_Y\}$, and (i) a simply test-wise deletion CI test yields $X \Vbar Y |\{\mathbf{Z},R_X,R_Y\mathbf{R_Z}\}$; or (ii) after the correction by Prop. \ref{lemma:recover}, the CI test yields $X \Vbar Y |\{\mathbf{Z},\mathbf{R_S}\}$. 
	
	For case (i), with the faithfulness assumption, $X \Vbar Y |\{\mathbf{Z},R_X,R_Y\mathbf{R_Z}\}$ indicates that $X$ and $Y$ is d-separated from $\mathbf{Z}$ in the complete data. Since $Z_i\in\mathbf{Z}$ and $Z_i$ is not self-masking variable, in the missing data $X\to R_{Z_i} \leftarrow Y$ is a collider and there must satisfy $X \not\Vbar Y |\{\mathbf{Z},R_X,R_Y,\mathbf{R_Z}\}$, which is contradiction.
	
	Similarly, for case (ii), based on Theorem \ref{thm:self_ind}, $X \Vbar Y |\{\mathbf{Z},\mathbf{R}_S\}$ indicates $X$ and $Y$ is d-separated from $\mathbf{Z}$ in the complete data. However, since $Z_i\in\mathbf{Z}$ and $Z_i$ is not self-masking variable, $X\to R_{Z_i} \leftarrow Y$ is a collider and there must satisfy $X \not\Vbar Y |\{Z,R_S\}$, which is contradiction. This finishes the proof.
\end{proof}

\subsection{Proof of Theorem 2}
\begin{theorem}[Identification of missing indicators]
	Suppose Assumption 1 $\sim$ 3 in \cite{tu2019causal} holds, and further assume weak self-masking missingness and structure condition hold. The causal relations of missing indicators are identifiable.
	% then the indicator of weak self-masking $Z$ is identifiable by Method \ref{method:Locate_Mis} if and only if there exist two variables $X$, $Y$ such that $X$ is d-separated from $Y$ given $Z$ in the ground truth.
\end{theorem}
\begin{proof}
	The correctness is straightforward based on Lemma \ref{lemma: self_indicator}. Based on the structure condition, there must exist variables $X,Y\in \mathbf{V}$ and $\mathbf{Z} \subseteq \mathbf{V}\setminus \{X,Y\}$ such that $X\Vbar_d Y|\mathbf{Z}$. When $Z_i\in \mathbf{Z}$ is the self-masking missingness variable, in such a case, it must satisfy ${X,Y}\not \Vbar R_{Z_i}|\{R_X,R_Y\}$, and either a simply test-wise deletion CI test will yield $X \Vbar Y |\{\mathbf{Z},R_X,R_Y,\mathbf{R_Z}\}$; or after the correction by Prop. \ref{lemma:recover}, the CI test yields $X \Vbar Y |\{\mathbf{Z},\mathbf{R}_S\}$. Thus, based on Lemma \ref{lemma: self_indicator}, $Z_i$ is the self-masking missingness variable.
\end{proof}

\subsection{Proof of Remark 1}
\begin{remark}
	Given an ANM $V_i=f_i(Pa_{V_i})+\varepsilon_i$ in complete data, one can identify the causal direction by testing the independence between the residuals of regression and the hypothesis cause, such that the independence holds only in the correct causal direction, e.g., $Pa_{V_i} \Vbar V_i - f_i(Pa_{V_i})$.
\end{remark}

\begin{proof}
	This proof can be referred to \citet{hoyer2008nonlinear} in the non-Linear case and \citet{shimizu2006linear} in the linear non-Gaussian case.
\end{proof}

\subsection{Proof of Theorem 3}
\begin{theorem}[Identifiability of ANM in missing data]
	Given an m-graph $G$, an additive noise model
	\begin{equation*}
		V_i=f_i(Pa_{V_i})+\varepsilon_i,\quad Pa_{X_i}\Vbar \varepsilon_i
	\end{equation*}
	is identifiable in complete data but not identifiable in missing data if and only if there exists a directed path in $G$ that starts from one of the parent $V_j \in Pa_{V_i}$ and ends at missing indicator $R_{V_j}$ or $R_{V_i}$.
\end{theorem}

\begin{proof}

	"If" part: In this part, we show that if there exists a directed path in $G$ that starts from one of the parent $V_j \in Pa_{V_i}$ and ends at missing indicator $R_{V_j}$ or $R_{V_i}$, then the independent noise will no longer hold, i.e., $Pa_{V_{i}}\not{\bot } \epsilon _{i}$. Without loss of generality, let $V_{k} \in \{V_{i} ,Pa_{V_{i}} \}$ whose missing indicator $R_{V_{k}}$ is the common descendent of $V_{i}$ and $Pa_{V_{i}}$, i.e., $R_{V_{k}} \in \operatorname{Des}_{V_{i}} \cap \operatorname{Des}_{Pa_{V_{i}}}$, then there exists a directed path from $V_k$ to $R_{V_k}$.
	
	To see the dependent noise, consider the conditional distribution $p(V_{i} |Pa_{V_{i}} )$. Since there exist a missing variable $V_{k}$, we can only observe the conditional distribution $p(V_{i} |Pa_{V_{i}} ,R_{V_{k}} )=p(\epsilon _{i} =V_{i} -Pa_{V_{i}} |Pa_{V_{i}} ,R_{V_{k}} )$. Then, By given the conditional set $R_{V_{k}}$, based on the faithfulness assumption, $Pa_{V_{i}}\not{\bot } \epsilon _{i} |R_{V_{k}}$, and thus $p(\epsilon _{i} =V_{i} -Pa_{V_{i}} |Pa_{V_{i}} ,R_{V_{k}} )\neq p(\epsilon _{i} =V_{i} -Pa_{V_{i}} |Pa_{V_{i}} )$ making ANM not identifiable.

	"Only If" part:
	In this part, we show that if there does not exist a missing variable $V_{k} \in \{V_{i} ,Pa_{V_{i}} \}$ whose missing indicator $R_{V_{k}}$ is the common descendent of $V_{i}$ and $Pa_{V_{i}}$, the ANM is identifiable. That is, there does not exist a directed path in $G$ that starts from $V_k$ and ends at missing indicator $R_{V_k}$.
	%$R_{V_{k}}\not{\in }\operatorname{Des}_{V_{i}} \cap \operatorname{Des}_{Pa_{V_{i}}}$, the ANM is identifiable. 
	Then, there must be one of the following two cases: (1) $R_{V_{k}} \bot \{V_{i} ,Pa_{V_{i}} \}$, (2) $R_{V_{k}}$ is the descendent of $Pa_{V_{i}}$ but not of $V_{i}$.
	
	For the first case, with the faithfulness assumption, $R_{V_{k}} \bot \{V_{i} ,Pa_{V_{i}} \}$ implies that $P(V_{i} ,Pa_{V_{i}} )=P(V_{i} ,Pa_{V_{i}} |R_{V_{k}} )$. The distribution of missing data is consistent with the complete distribution. Therefore, we conclude that ANM is identifiable according to \citet{peters2014causal}.
	
	For the second case, without loss generality, we let $V_{i}\leftarrow Pa_{V_{i}} \rightsquigarrow R_{Pa_{V_{k}}}$ such that $\displaystyle R_{Pa_{V_{k}}}$ is the descendent of $Pa_{V_{i}}$ and $R_{V_k}\in R_{Pa_{V_{k}}}$. The reason is that $R_{Pa_{V_{k}}}$ can not be the descendent of $V_i$ otherwise there will exist a directed from $V_k$ to $R_{V_k}$. Thus, in this case, we have $V_{i} \bot R_{Pa_{V_{k}}} |Pa_{V_{i}}$ according to the faithfulness assumption. Then $p(V_{i} |Pa_{V_{i}} )=p(V_{i} |Pa_{V_{i}} ,R_{Pa_{V_{k}}} )$ holds. For the causal direction, with the above result, we have $p(V_{i} |Pa_{V_{i}} ,R_{Pa_{V_{k}}} )=p(V_{i} |Pa_{V_{i}} )P(Pa_{V_{i}} |R_{Pa_{V_{k}}} )$. Therefore, the independence holds between the residuals of the regression of $V_{i}$ on $Pa_{V_{i}}$ and $Pa_{V_{i}}$. For the reverse causal direction, according to \citet{peters2014causal} and the faithfulness assumption, the independence does not hold between residuals and their hypothesis parents variables. Therefore, the ANM is identifiable.
\end{proof}

\subsection{Proof of Theorem 4}
\begin{theorem}[Orientation Rule]
	Given an IN-equivalent pattern represented by a partially directed m-graph $G(\mathbb{V},\mathbf{E} )$, where edges $\mathbf{E}$ may contain directed edges and undirected edges. If the potential non-identifiable paths of $V_i$ is empty, then we orient every undirected neighbor $V_j \in \{V_j|V_j - V_i\}$ as $V_i\to V_j$.
\end{theorem}
\begin{proof}
	We prove by contradiction. Without loss of generality, suppose an undirected neighbor $V_j$ of $V_i$ in the IN-equivalent pattern has the causal relationship $V_j\to V_i$ in the ground truth and $V_i$ has no potential non-identifiable paths. 
	
	Since $V_j$ and $V_i$ is not a directed edge in the IN-equivalent pattern, then, based on the definition of IN-equivalent pattern, there must have $Pa_{V_i}\not{\bot} V_i-f(Pa_{V_i})|R_{V_i},\mathbf{R}_{Pa_{V_i}}$, where $V_j\in Pa_{V_i}$, which is a not identifiable ANM in missing data. Based on Theorem \ref{thm:missing_anm_condition}, there must exist a directed path in $G$ that starts from one of the parent $V_j\in Pa_{V_i}$ and ends at missing indicator $R_{V_j}$ or $R_{V_i}$. Such a path must be one of the potential non-identifiable paths which contradicts to no potential non-identifiable paths. This finishes the proof. 
\end{proof}

\subsection{Proof of Theorem 5}
\begin{theorem}[Soundness of LCS-MD]
	Suppose that the data over variables $\mathbf{V}$ was generated by ANM and assumptions 1 $\sim$ 3 in \cite{tu2019causal}, assumptions of weak self-masking missingness and structural condition hold. Let $G$ denote the output of LCS-MD, all directed edge in $G$ is correctly oriented.
\end{theorem}
\begin{proof}
	Due to our identifiability result of causal direction is based on the correct causal skeleton and causal relationship of missing indicator, we first prove the correctness of SM-MVPC.
	
	\paragraph{The correctness of SM-MVPC} As shown in Algorithm 1, the causal skeleton can be reconstructed correctly by testing CI relations among the observed variables, which is already proved by \cite{spirtes2000causation}. According to Prop. \ref{lemma:recover} and Theorem 1, one can test the CI relations in the recovered distribution such that the CI relations are consistent with the original distribution. Thus, the causal skeleton can be learned correctly. 
	
	To determine the missing indicator, one can identify the missing indicator by the method provided in Lemma 1 (Line 2 $\sim$ 3 in Algorithm 1), and the correctness is proved by Theorem 2. By this, the correctness of SM-MVPC is proved.
	
	\paragraph{The correctness of causal direction} Now, we prove the correctness of the identified causal direction. For the causal direction oriented by Line 1 $\sim$ 4 in LCS-MD, the correctness is ensured by the identifiability of ANM in missing data (Theorem 3). In line 5 $\sim$ 6, the causal direction is identified by orientation rules, which also is proved in Theorem 4.
	
	According to the above proofs, we finish the proof of the correctness of causal directions in the output of LCS-MD.
	
\end{proof}

\section{Supplementary experiments}
The main paper has shown the F1 scores and other baselines in synthetic experiments. Here, we further provide the Recall, Precision, and SHD metrics for all these experiments in Figure \ref{fig:sensitivity_recall}$\sim$\ref{fig:rule_full}. 

One can see that our method has better precision and recall than baselines in all cases. 

\begin{figure*}[h!]
	\centering
	\subfigure[Sensitivity of sample size]{
		\includegraphics[width=0.3\textwidth]{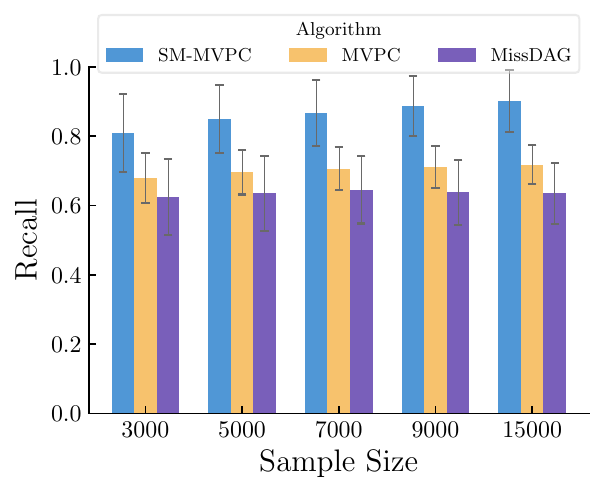}
		\label{fig:ske:Sample_SizelinearSMNARRecall}
	}
	\subfigure[Sensitivity of the weak self-masking]{
		\includegraphics[width=0.3\textwidth]{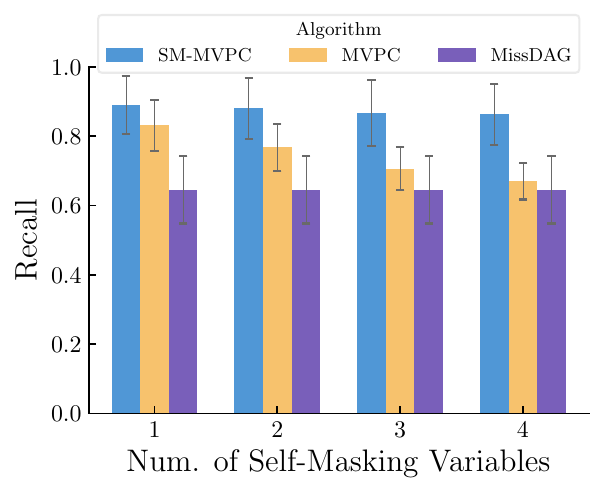}
		\label{fig:ske:Num_of_self_nodelinearSMNARRecall}
	}
	\caption{Recall in Experiments for Skeleton Learning}	
	\label{fig:sensitivity_recall}
\end{figure*}

\begin{figure*}[h!]
	\centering
	\subfigure[Sensitivity of sample size]{
		\includegraphics[width=0.3\textwidth]{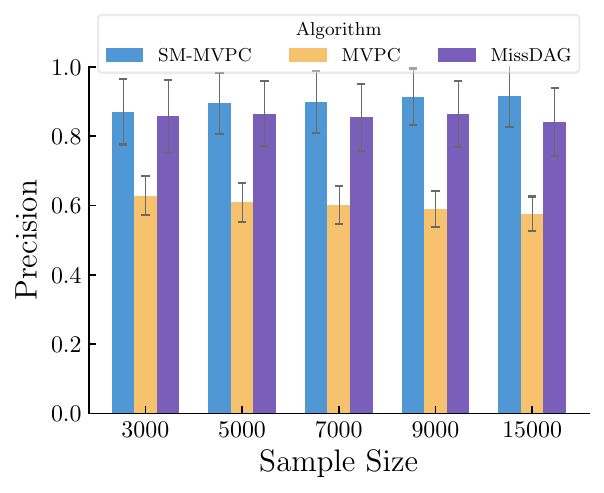}
		\label{fig:ske:Sample_SizelinearSMNARPrecision}
	}
	\subfigure[Sensitivity of the weak self-masking]{
		\includegraphics[width=0.3\textwidth]{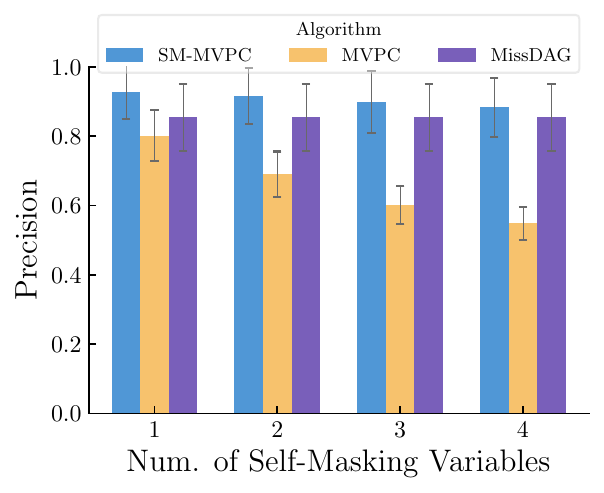}
		\label{fig:ske:Num_of_self_nodelinearSMNARPrecision}
	}
	\caption{Precision in Experiments for Skeleton Learning}	
	\label{fig:sensitivity_precision}
\end{figure*}

\begin{figure*}[h!]
	\centering
	\subfigure[Sensitivity of sample size]{
		\includegraphics[width=0.3\textwidth]{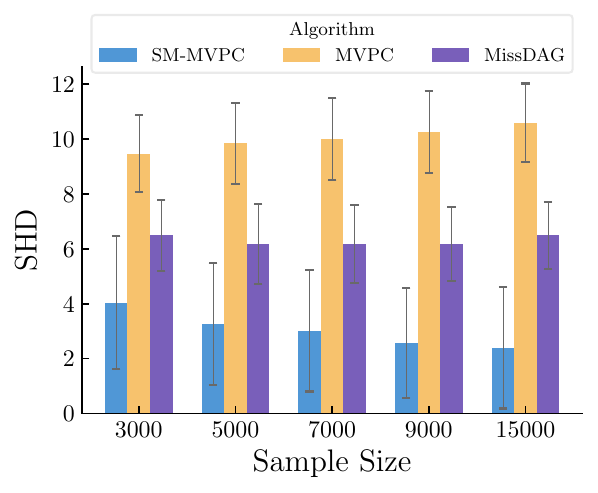}
		\label{fig:ske:Sample_SizelinearSMNARSHD}
	}
	\subfigure[Sensitivity of the weak self-masking]{
		\includegraphics[width=0.3\textwidth]{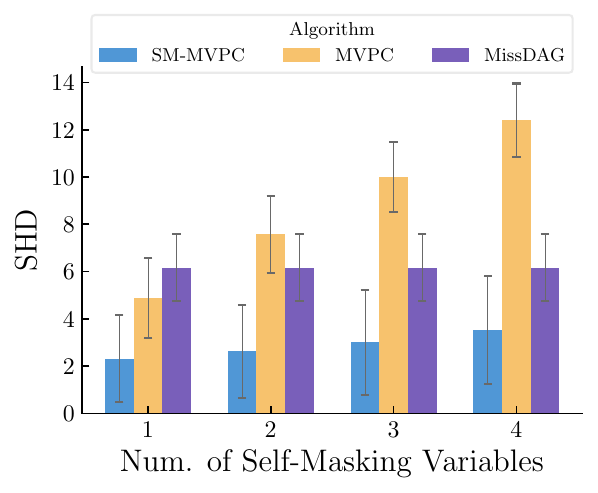}
		\label{fig:ske:Num_of_self_nodelinearSMNARSHD}
	}
	\caption{SHD in Experiments for Skeleton Learning}	
	\label{fig:sensitivity_shd}
\end{figure*}

\begin{figure*}[h!]
	\centering
	\subfigure[F1]{
		\includegraphics[width=0.3\textwidth]{images/2024/ruleF1.pdf}
	}
	\subfigure[Recall]{
		\includegraphics[width=0.3\textwidth]{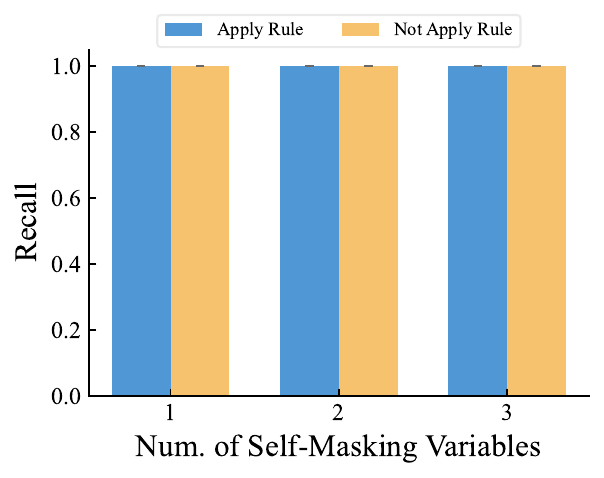}
	}
	\subfigure[Precision]{
		\includegraphics[width=0.3\textwidth]{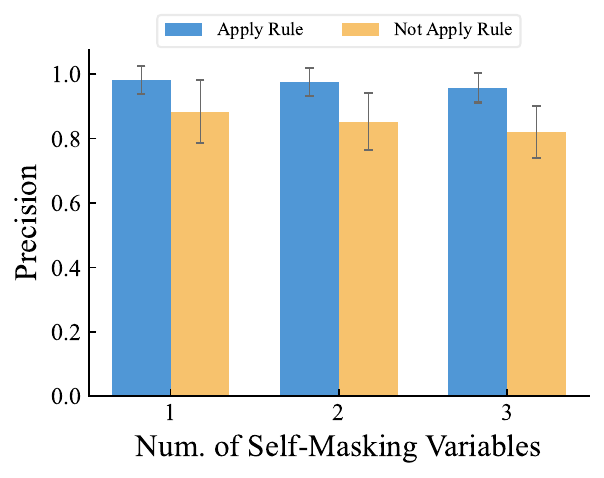}
	}
	\caption{Applying rule with correct skeleton}
	\label{fig:rule_full}
\end{figure*}

%\bibliography{aaai24.bib}

\end{document}